\documentclass{article}

\usepackage[T1]{fontenc}

\usepackage{microtype}
\usepackage{graphicx}
\usepackage{subcaption}
\usepackage{booktabs} % for professional tables

\usepackage{hyperref}

\usepackage[preprint]{arxiv2026}

\usepackage{amsmath}
\usepackage{amssymb}
\usepackage{mathtools}
\usepackage{amsthm}

\usepackage[capitalize,noabbrev]{cleveref}

\usepackage{fancyvrb}
\usepackage{fvextra}  % extends fancyvrb with breaklines support
\usepackage{quoting}  % for breakable quote environment
\fvset{commandchars=\\\{\},samepage=false}

\usepackage{tcolorbox}
\tcbuselibrary{breakable}

\usepackage{placeins}
\usepackage{float}
\usepackage{multirow}
\usepackage{afterpage}

\usepackage{listings}
\usepackage{xcolor}
\theoremstyle{plain}
\newtheorem{theorem}{Theorem}[section]
\newtheorem{proposition}[theorem]{Proposition}

\theoremstyle{definition}

\theoremstyle{remark}

\usepackage[textsize=tiny]{todonotes}

\icmltitlerunning{Canonical Intermediate Representation for LLM-based optimization}

\begin{document}

\raggedbottom

\twocolumn[
  \icmltitle{Canonical Intermediate Representation for \\
  LLM-based optimization problem formulation and code generation}

  % the [accepted] option to the arxiv2026 package.



  \begin{icmlauthorlist}
    \icmlauthor{Zhongyuan Lyu}{polyu}
    \icmlauthor{Shuoyu Hu}{polyu}
    \icmlauthor{Lujie Liu}{polyu}
    \icmlauthor{Hongxia Yang}{polyu,infix}
    \icmlauthor{Ming LI}{polyu}
  \end{icmlauthorlist}

  \icmlaffiliation{polyu}{The Hong Kong Polytechnic University, Hong Kong, China}
  \icmlaffiliation{infix}{InfiX.ai, Hong Kong, China}

  \icmlcorrespondingauthor{Ming LI}{ming.li@polyu.edu.hk}

  \icmlkeywords{Machine Learning, Optimization, Code Generation}

  \vskip 0.3in
]



\printAffiliationsAndNotice{}  % no special notice (required even if empty)

\begin{abstract}
  Automatically formulating optimization models from natural language descriptions 
  is a growing focus in operations research, yet current LLM-based approaches 
  struggle with the composite constraints and appropriate modeling paradigms 
  required by complex operational rules. To address this, we introduce the 
  Canonical Intermediate Representation (CIR): a schema 
  that LLMs explicitly generate between problem descriptions and 
  optimization models. 
  CIR encodes the semantics of operational rules through constraint archetypes 
  and candidate modeling paradigms, thereby decoupling rule logic 
  from its mathematical instantiation. 
  Upon a newly generated CIR knowledge base, 
  we develop the rule-to-constraint (R2C) framework, a multi-agent pipeline that parses 
  problem texts, synthesizes CIR implementations by retrieving domain knowledge, and instantiates optimization models. 
  To systematically evaluate rule-to-constraint 
  reasoning, we test R2C on our newly constructed benchmark featuring rich 
  operational rules, and benchmarks from prior work. 
  Extensive experiments show that R2C achieves state-of-the-art accuracy 
  on the proposed benchmark (47.2\% Accuracy Rate). 
  On established benchmarks from the literature, R2C delivers highly competitive 
  results, approaching the performance of proprietary models (e.g., GPT-5). 
  Moreover, with a reflection mechanism, R2C achieves further gains and sets 
  new best-reported results on some benchmarks.

\end{abstract}

\section{Introduction}

\begin{figure*}[t]
  \centering
  \includegraphics[width=0.85\textwidth]{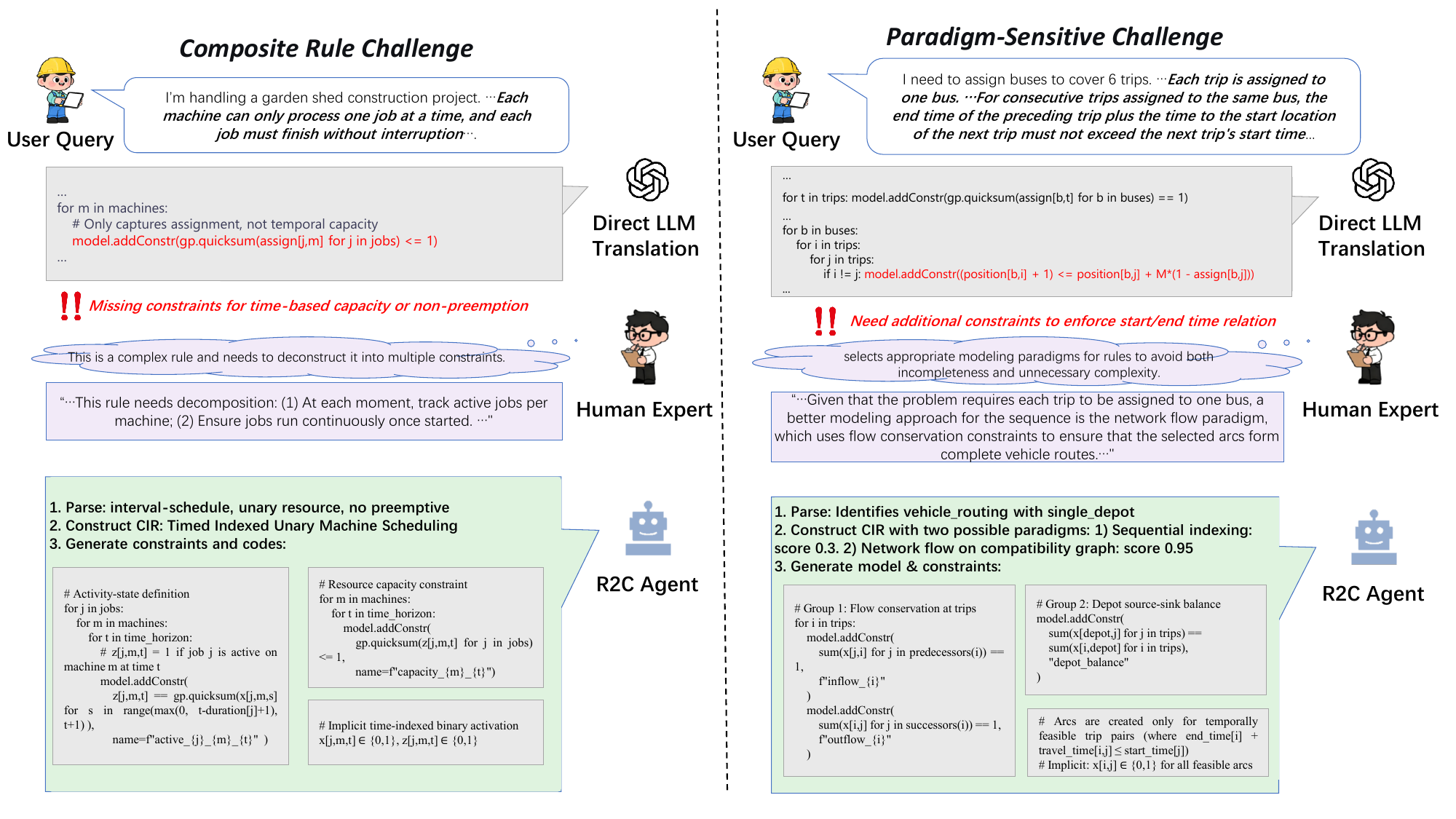}
  \caption{
    Examples of LLM-based Automated Optimization Formulation with 
    Complex Operational Rules. The figure presents a comparative analysis 
    of Direct LLM Translation, Human Expert Reasoning, and our R2C 
    Framework in addressing two core challenges: the Composite Rule 
    challenge (left), where a single rule decomposes into multiple 
    constraints; and the Paradigm-Sensitive challenge (right), 
    where the mathematical formulation is governed by the chosen 
    modeling paradigm.
    }
  \label{fig:comparison}
\end{figure*}
Operations research (OR) aims to 
improve the performance and efficiency of complex systems across a 
wide range of industrial and scientific contexts.
However, translating real-world optimization 
problems into precise mathematical models remains a primary challenge, requiring both 
expertise in optimization theory and domain-specific knowledge \cite{ramamonjison2022augmenting}. The emergence of large language models (LLMs) 
enables the automation of this process, making LLM-driven optimization and code 
generation an important research direction with profound implications for both 
academia and industry. 

Recent advances have demonstrated the capability of LLMs to generate optimization 
models and code from natural language descriptions through comprehensive approaches, 
including direct prompting \cite{bertsimas2024robust,li2023large}, multi-agent 
systems \cite{ahmaditeshnizi2024optimus,mostajabdaveh2024optimization}, and post-training on 
domain-specific data \cite{huang2025orlm,jiang2024llmopt}. Despite incorporating reflection or feedback mechanisms, existing approaches 
still rely on direct text-to-constraint conversion, 
facing critical challenges in complex operational rules.

In both industrial and academic settings, 
optimization problems are typically described in terms of operational 
rules (e.g., "Each aircraft must undergo maintenance every 100 flight hours."). 
Critically, the translation of these rules into mathematical constraints 
presents two fundamental challenges, as illustrated in \cref{fig:comparison}: 
(1) compositeness—a single rule often requires decomposition into 
multiple constraints, and (2) paradigm-sensitivity—the 
same rule may demand different formulations depending on 
the chosen modeling paradigm (e.g., discrete-time capacity constraints 
vs. continuous-time disjunctive constraints).

While LLMs can generate most correct constraints for a given 
problem, they often fail to model complex operational rules that 
require decomposition and may produce paradigm-incompatible 
formulations without specialized tuning. 
Such localized errors can invalidate the overall model
In contrast, human experts explicitly decompose rules and select modeling paradigms 
by reasoning over rule interactions and problem structure.
A further limitation is that existing benchmarks for LLM-based optimization modeling 
focus on broad coverage across problem categories 
\cite{huang2025orlm,lu2025optmath} or 
specific problem classes (e.g., dynamic programming 
\cite{zhou2025auto}), but lack instances with complex 
operational rules, making it 
hard to evaluate LLMs' capabilities in reasoning about complex 
operational rule semantics and modeling paradigms.

To address these challenges, we introduce the Canonical Intermediate 
Representation (CIR): a machine-readable schema that LLMs explicitly 
generate between natural-language problem descriptions and optimization 
models. 
CIR encodes the semantics of operational rules through constraint archetypes and candidate modeling paradigms
Building on CIR, we develop Rule-to-Constraint (R2C), a training-free, 
multi-agent framework that treats automatic formulation as a structured 
pipeline. R2C first parses problem text 
into structured rule representations, then uses retrieval-augmented 
generation (RAG) over a constructed CIR-based knowledge base to synthesize 
instance-specific CIR objects and perform paradigm-oriented clustering, and 
finally instantiates a complete mathematical model and executable solver 
code. 

We evaluate R2C on our newly constructed Operational Rule to Constraint Optimization
Benchmark (ORCOpt-Bench), specifically designed to test complex rule-to-constraint reasoning, 
and on benchmarks from prior work. 
Extensive experiments demonstrate that R2C achieves 
state-of-the-art performance on ORCOpt-Bench, outperforming strong baselines including 
the proprietary models GPT-5 \cite{openai2025gpt5} and Grok-4 \cite{xai2025grok4}. 
The framework also achieves competitive results on established benchmarks, 
matching the performance of leading closed-source models 
while demonstrating generalization without task-specific training.
Furthermore, when extended with a reflection mechanism, R2C attains additional 
performance gains, and achieves state-of-the-art performance on some benchmarks.

Our main contributions are:
\begin{enumerate}
\item We reframe LLM-based optimization modeling from direct constraint 
translation to the structured representation of rule semantics and paradigm 
selection, targeting the challenge of complex rule-to-constraint mapping.

\item We propose CIR, a structured intermediate representation that 
decouples operational logic from mathematical instantiation and 
is explicitly generated by LLMs.

\item We develop a robust, multi-agent pipeline that operationalizes 
CIR via RAG, offering a training-free solution that integrates seamlessly with off-the-shelf LLMs.

\item We construct ORCOpt-Bench to evaluate complex rule-to-constraint reasoning, 
and show that R2C achieves state-of-the-art accuracy on it 
and delivers highly competitive results on established benchmarks.

\end{enumerate}

\section{Related work}

\subsection{Application of LLMs in Optimization Problem}

Applications of LLMs to optimization problems roughly follow two paths. 
The first uses LLMs to assist the design and refinement of optimization algorithms, 
for example by evolving heuristics or tuning solver 
components \cite{ye2024reevo,liu2024evolution,sun2024autosat}. 
The second, which is the focus of this work, employs LLMs to translate natural-language 
descriptions of decision problems into formal models and executable code.

Early explorations in the latter direction directly exploited the pre-trained 
knowledge of general-purpose LLMs via prompting. \citet{bertsimas2024robust} 
showed that ChatGPT can formulate robust and adaptive counterparts from classical 
optimization problem statements, while \citet{li2023large,chen2025optichat} 
developed dialog systems that allow users to query and interpret optimization 
models and solutions in natural language.

Recognizing the difficulty of end-to-end formulation, subsequent work decomposes 
the process into specialized subtasks handled by multi-agent systems. 
OptiMUS \cite{ahmaditeshnizi2024optimus}, for example, automates mixed-integer 
linear programming (MILP) formulation, and OR-LLM-Agent \cite{zhang2025or} allocates 
modeling, coding, and debugging to different LLM agents. In parallel, data-centric 
approaches adopt supervised fine-tuning to improve domain fidelity. ORLM \cite{huang2025orlm} 
and LLMOpt \cite{jiang2024llmopt} construct large corpora of optimization problems and 
formulations to fine-tune smaller LLMs (7B to 32B). 
These SFT-based methods can achieve high accuracy but generally require 
substantial labeled data and computational resources.

Despite this progress, existing approaches rely on direct rule-to-constraint 
translation, struggling with complex rule interactions. Instead of solely improving 
single-step generation, we address this fundamental structural gap through an 
explicit semantic abstraction step.

\subsection{RAG applications in LLMs}

RAG augments LLMs with dynamically retrieved external knowledge to mitigate 
issues such as hallucinations and missing domain information. For a 
comprehensive review of this field, readers are referred to recent surveys
\cite{fan2024survey,zhao2024retrieval,gao2023retrieval}. It has been widely used in 
code generation and mathematical reasoning. For code, retrieval-augmented generation 
improves tasks such as code synthesis, completion, 
and repair \cite{parvez2021retrieval,ahmad2021unified,lu2022reacc,joshi2023repair}. 
For mathematical reasoning, RAG helps theorem provers and QA systems by retrieving 
relevant premises and textbook material \cite{yang2023leandojo,levonian2023retrieval}.

Recently, RAG has been adapted to optimization problem solving. 
DRoC \cite{jiang2025droc} tackles vehicle routing problems by decomposing 
pre-identified constraints and retrieving solver documentation and code for 
each constraint, dynamically balancing retrieval and self-debugging. 
LEAN-LLM-OPT \cite{liang2025llm} orchestrates multiple LLM agents to retrieve 
modeling templates and examples for large-scale formulations.

While RAG has advanced LLM-based optimization automation, a core 
disconnect remains between semantic rule understanding and constraint 
formulation: prior methods either bypass the translation from natural 
language (e.g., DRoC \cite{jiang2025droc}) or rely on rigid templates 
unsuitable for complex interactions (e.g., LEAN-LLM-OPT \cite{liang2025llm}). 
We address this by introducing CIR, a structured semantic 
layer that decouples operational intent from mathematical encoding, 
enabling systematic rule-to-constraint mapping.

\section{Method}
\subsection{Problem Definition}
We begin by formalizing the task of automatic optimization modeling and code generation from natural language. 
An optimization problem is specified by text $d \in \mathcal{D}$ that 
specifies: operational rules $\mathcal{R}(d) = \{r_1, \ldots, r_K\}$, 
entities, parameters, and an objective. Given the problem description $d$, 
we seek to produce: a mathematical model $M(d)$, and an executable solver code $\text{code}(d)$ that, 
when solved, yields an optimal decision respecting all operational rules.

\subsection{Canonical Intermediate Representation}
The CIR is a structured representation that decouples the semantic intent 
of operational rules from their mathematical formulation.
The CIR is built upon a predefined knowledge base of templates, 
which are then instantiated into concrete implementations during LLM-based reasoning. 
Each CIR template in the knowledge base encodes a core operational intent and 
its possible realizations under multiple modeling paradigms. Formally, a template comprises:

\begin{itemize}
  \item The semantic operational intent.
	\item The list of supported paradigms under which the semantic intent can be formulated 
  (e.g., {continuousTime, timeIndexed}).
  \item For each paradigm $p$, the template provides a mathematical expression 
  that instantiates the intent. These expressions use canonical placeholders for indices, decision variables, and parameters.
\end{itemize}

Details of the curated CIR knowledge base are provided in Appendix~\ref{app:cir}. In essence, a CIR template is not a single constraint but a bundle of semantically equivalent, paradigm-specific formulations. 
The proposed framework employs a library of predefined CIR templates, and 
each operational rule $r_k \in \mathcal{R}(d)$ is mapped 
to one or more CIR implementations. This mapping is achieved either by instantiating 
relevant CIR templates or by leveraging the LLM's internal knowledge. 
Each implementation takes the form $a_\ell = (A_\ell, k_\ell, p_\ell)$, where:
\begin{itemize}
	\item $A_\ell$ identifies the core intent (e.g., \texttt{NonOverlap}),
	\item $k_\ell$ points to the source rule $r_k$,
	\item $p_\ell$ represents the selected modeling paradigm. The implementation 
  instantiates the template by binding its placeholders 
  to the entities and parameters of problem $d$. While an implementation 
  may support multiple paradigms (\autoref{fig:r2c-framework}), the framework selects 
  a final one that is compatible across all rules.
\end{itemize}

\paragraph{From CIR to Mathematical Model.}
During LLM-based generation for a given problem $d$, all CIR implementations are aggregated into a unified 
CIR for the problem: $C(d) = (\mathcal{E}(d), \mathcal{A}(d))$, 
where $\mathcal{E}(d)$ contains all entities/parameters 
and $\mathcal{A}(d) = \{a_1, \ldots, a_L\}$ is the multiset of CIR implementations.
For each implementation $a_\ell$ with the intent $A_\ell$ and paradigm $p_\ell$, the framework generates the concrete 
mathematical constraints $\mathcal{C}_{A_\ell, p_\ell}$. Then, the complete model is the union of these constraints:
\begin{equation}
M(d) = T(C(d)) = \bigcup_{\ell=1}^{L} \mathcal{C}_{A_\ell, p_\ell}.
\label{eq:CIR-to-model}
\end{equation}

CIR is grounded in the paradigm of augmenting LLMs with structured knowledge, 
moving beyond unstructured text to abstract representations \citep{fevry2020entities, rubin2022learning}. 
By retrieving CIR templates, the framework focuses on relevant information and 
ensures semantic fidelity to the original rules, ensuring any solution satisfying 
the model respects all operational constraints. 
This property is formally captured by the following soundness guarantee:

\begin{proposition}[CIR Soundness Guarantee]
\label{prop:cir-soundness-main}
Under the design of the CIR knowledge base, every feasible solution of the translated model $M(d)$ satisfies all original operational rules in $\mathcal{R}(d)$. Formally,
\[
\mathcal{F}(M(d)) \subseteq \mathcal{F}_{\text{rules}}(d).
\]
\end{proposition}
The proof, which relies on the semantic equivalence between rules and their CIR implementations 
and the soundness of the template instantiation process, 
is provided in Appendix~\ref{app:proof-soundness}.

\subsection{R2C Agent Framework}
\begin{figure*}[t]
  \centering
  \includegraphics[width=0.85\textwidth]{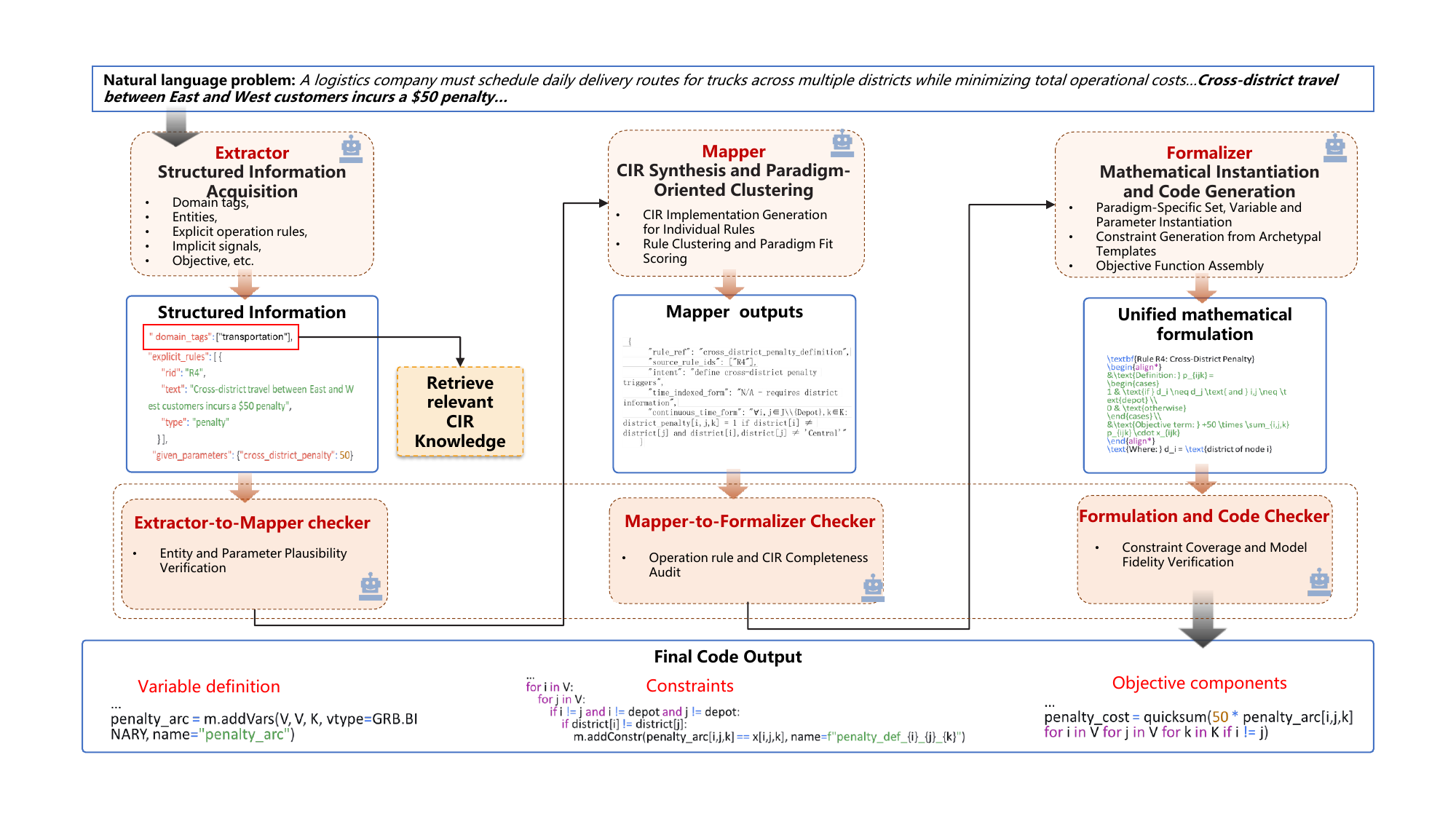}
  \caption{The R2C agent framework architecture. } 
  \label{fig:r2c-framework}
\end{figure*}

In this subsection, we formally propose the R2C agent framework, which 
is shown in \cref{fig:r2c-framework}. The R2C framework consists of four specialized agents—Extractor, 
Mapper, Formalizer, and Checker—operating across three stages: 
structured information extraction, CIR synthesis, and mathematical 
formulation with code generation. The Extractor first parses 
the problem description $d$ into structured entities, rules, and parameters. 
The Mapper then adaptively instantiates CIRs for each rule using the domain CIR library. 
The Formalizer assembles these into a complete mathematical model, defining sets, 
variables, constraints, and the objective. To prevent error propagation, 
the Checker validates each agent's output before proceeding.

\subsubsection{The Extractor Agent: Structured Information Acquisition}

The Extractor Agent parses 
the problem description $d$ and extracts structured 
information into a structured schema. 
This process reduces the semantic ambiguity inherent in handling complex operational rules, 
in contrast to direct processing of unstructured narratives
\citep{mostajabdaveh2024optimization, ahmaditeshnizi2024optimus}.
It prompts the LLM to identify and tag key elements—enabling verifiable, 
end-to-end traceability \citep{jiang2024llmopt}—and populates the following components:(1) \textit{Domain 
tags} for retrieving relevant CIR knowledge; 
(2) \textit{Entities} (e.g., tasks/jobs, resources) with their 
key attributes; (3) 
\textit{Explicit operation rules}, each tagged with a 
ID for traceability; (4) \textit{Implicit signals}, inferred to 
mitigate common modeling omissions; and (5) \textit{Objective}, 
decomposed into its constituent cost or utility components. The complete Extractor prompt is detailed in \autoref{app:extractor}. 

\subsubsection{The Mapper Agent: CIR Synthesis and Paradigm-Oriented Clustering}
The Mapper Agent is tasked with transforming the structured output 
from the Extractor into a formalization-ready schema. The Mapper has two core tasks:
1. \textbf{CIR Implementation Generation for Individual Rules:} 
To address the compositeness challenge, the Mapper maps each operational rule to one or more related CIR implementations.
Given an operation rule, the Mapper queries the domain-specific CIR library 
according to the problem domain to retrieve compatible CIR archetypes \citep{asai2024self, zhao2024retrieval}. 
It then re-parameterizes the retrieved archetypes to adapt them to the problem context through 
operations such as lifting dimensionalities, tightening coefficients via bound propagation, or introducing auxiliary variables.
2. \textbf{CIR Clustering and Paradigm Fit Scoring:} 
To directly address the paradigm-sensitivity challenge, 
the Mapper clusters CIR implementations by shared structural needs 
(e.g., network flow vs. disjunctive scheduling) and assigns a 
Paradigm Fit Score to each candidate paradigm. 
The highest-scoring paradigm is selected for consistent model instantiation.
By maintaining both 
the detailed CIR implementations and the compatible paradigm, 
the framework provides full traceability for each rule's modeling 
process \citep{gao2023retrieval}. The complete Mapper prompt is detailed in \autoref{app:mapper}.

\subsubsection{The Formalizer Agent: Mathematical Instantiation of CIRs}
The Formalizer then integrates the structured output from the Extractor, Mapper, 
and the original problem description. Leveraging the LLM's inherent capabilities, 
it synthesizes these inputs into a precise mathematical model and the corresponding 
executable, solver-specific code (e.g., Gurobi Python).
The Formalizer prompt is detailed in \autoref{app:formalizer}.

\subsubsection{The Checker Agent: Hierarchical Verification for Robust Modeling}

The Checker Agent performs staged validation to prevent error propagation caused 
by LLM randomness and performance fluctuations of LLM services. The validation 
procedure is designed around three phases:
1. Extractor-to-Mapper phase: it checks the internal consistency of extracted 
data (e.g., time windows, resource demands); 
2. Mapper-to-Formalizer phase: it audits whether the set of CIR implementations for each 
operation rule is complete; 
3. Mathematical formulation and code checking phase: it verifies that 
every explicit and validated implicit operation rules have been 
translated into corresponding mathematical constraints without any omission. 
The detailed prompts are provided in \autoref{app:checker-prompts}.

\section{Numerical experiments}

\subsection{Experimental setup}
\label{sec:experimental-setup}

In this subsection, we first introduce the benchmarks used for evaluation, 
including our newly constructed one, then detail the baseline methods and 
performance metrics. The constructed benchmark, the implementation of the R2C framework, along 
with an anonymized web-based demo, is provided in the supplementary material.

\subsubsection{Benchmarks}
To systematically evaluate the ability of LLMs to reason about mapping complex operational rules into constraints, 
we introduce the ORCOpt-Bench, specifically designed to 
test an LLM's ability to handle intricate rule interactions. 
We created 50 optimization problems from academic literature (distinct from the literature used to construct the CIR knowledge base) 
and industrial 
case studies across 10 diverse domains. The detailed illustration is given 
in Appendix~\ref{app:orcoopt}.

To comprehensively evaluate the R2C framework's performance and generalizability, 
we complement the evaluation with three additional well-established 
benchmarks which contain relatively challenging optimization instances: 1) IndustryOR \citep{huang2025orlm}; 2) BWOR \citep{zhang2025or}; 3) OptMATH \citep{lu2025optmath}.

\subsubsection{Baseline Methods}
We compare R2C against a wide range of representative baselines: 
Chain-of-Experts \citep{xiao2023chain}, Reflexion \citep{shinn2023reflexion}, Vanilla RAG, and Self-RAG \citep{asai2024self}. 
For a controlled and fair comparison, the aforementioned 
baseline methods are implemented using the same backbone LLM DeepSeek-v3.2. 
For broader context, we also report results from proprietary LLMs 
(GPT-5 \citep{openai2025gpt5}, Grok-4 \citep{xai2025grok4}) using 
standard prompting (\autoref{app:standard-prompt}). 

We also tried directly prompt the LLM to handle rule decomposition and 
paradigm compatibility. However, this ablated approach typically degrades performance
and was thus excluded it in the comparison.

The R2C framework employs DeepSeek-v3.2 and Qwen3-32B 
models to evaluate performance across different model scales.

\subsubsection{Performance Metrics}
We employ the following two widely adopted metrics to quantitatively assess the performance of all methods:

•	Accuracy Rate (AR): 
measures the proportion of generated codes that yield a correct solution. 
A solution is correct if its objective value falls within 1\% of the reference optimum, 
accounting for numerical and solver tolerances.

•	Execution Rate (ER): measures the percentage of codes that execute 
successfully without errors. 

For a fair and reproducible comparison under the pass@1 metric, all models in our direct 
evaluations were run using the LLMs' default inference parameters to ensure compatibility 
across different model families. We conducted five independent runs for each evaluation and 
report the average performance to ensure statistical stability. For results taken 
from literature (reference methods in \cref{tab:external-benchmarks}), we report them as originally published.

\begin{table}[t]
  \centering
  \caption{Performance comparison on ORCOpt-Bench. AR and ER denote Accuracy Rate and Execution Rate, respectively.}
  \label{tab:performance-comparison}
  \begin{tabular}{lcc}
    \toprule
    Method & AR (\%) & ER (\%) \\
    \midrule
    Standard prompt (Qwen3 32B) & 8 & 36.6 \\
    Standard prompt (deepseek-v3.2) & 22.4 & 74.8 \\
    Chain-of-Experts & 30.2 & 80.6 \\
    Reflexion & 37.8 & 85.2 \\
    Vanilla RAG-raw & 28 & 78.4 \\
    Self-RAG-raw & 30.4 & 78.2 \\
    GPT5 & 39.8 & 80.8 \\
    Grok4 & 45.4 & 87.4 \\
    R2C (Qwen3-32B) & 21.2 & 64 \\
    R2C (deepseek-v3.2) & 47.2 & 83.2 \\
    \bottomrule
  \end{tabular}
\end{table}

\begin{table}[t]
  \centering
  \caption{Domain-wise performance comparison on ORCOpt-Bench. Numbers indicate the average count of correctly solved problems per domain.}
  \label{tab:domain-performance}
  \begin{tabular}{lcc}
    \toprule
    Domain & R2C & Standard prompt \\
    \midrule
    Crew Assignment & 2.6 & 0.8 \\
    Education & 3.8 & 3.2 \\
    Energy System & 2.4 & 1.0 \\
    Healthcare Allocation & 1.2 & 0.0 \\
    Supply Chain \& Production & 0.0 & 0.0 \\
    Project Management & 2.6 & 0.0 \\
    Resource Allocation & 3.2 & 2.0 \\
    Sports Planning & 3.2 & 3.0 \\
    Job Shop & 1.6 & 0.0 \\
    Transportation & 3.0 & 1.2 \\
    \bottomrule
  \end{tabular}
\end{table}
\begin{table*}[h]
  \centering
  \caption{Accuracy Rate (AR) comparison on external benchmarks (\%). Results marked with -- indicate that the method was not evaluated on that benchmark.}
  \label{tab:external-benchmarks}
  \begin{tabular}{lccccc}
    \toprule
    Method & Backbone & IndustryOR & BWOR & OptMATH \\
    \midrule
    GPT-4 & -- & 28.0 & 40.2 & 16.6 \\
    ORLM \cite{huang2025orlm} & LLaMa3-8B & 38.0 & 29.3 & -- \\
    OR-LLM-Agent \cite{zhang2025or} & DeepSeek-R1 & 36.0 & 82.9 & -- \\
    OptMATH \cite{lu2025optmath} & Qwen2.5-32B & -- & -- & 34.7 \\
    Standard Prompt & deepseek-v3.2 & 53 & 63.4 & 44.0 \\
    R2C & deepseek-v3.2 & 57.0 & 75.6 & 44.6 \\
    GPT5 & -- & 59.0 & 78.0 & 45.8 \\
    Grok4 & -- & 61.0 & 80.5 & 47.0 \\
    \bottomrule
  \end{tabular}
\end{table*}

\begin{table}[h]
  \centering
  \caption{Ablation study results on ORCOpt-Bench.}
  \label{tab:ablation}
  \begin{tabular}{lcc}
    \toprule  
    Configuration & AR (\%) & ER (\%) \\
    \midrule
    Base & 22.4 & 74.8 \\
    Full R2C & 47.2 & 83.2 \\
    R2C w/o CIR & 31.6 & 77.4 \\
    Vanilla RAG-CIR & 37.9 & 81.2 \\
    Self-RAG-CIR & 38.4  & 82.0 \\
    R2C w/o Validation & 41.6 & 81.8 \\
    \bottomrule
  \end{tabular}
\end{table}

\subsection{Performance Comparison on ORCOpt-Bench}
\label{sec:performance-comparison}
This experiment provides a comprehensive performance comparison between 
the proposed R2C framework and various baseline methods on ORCOpt-Bench. 
The Vanilla RAG-raw and Self-RAG-raw methods retrieve information 
directly from the raw domain literature that was used to construct the 
CIR knowledge base. In the ablation analysis (\cref{sec:ablation}), 
to evaluate the added value of our structured CIR abstraction, 
we implement two additional variants: Vanilla RAG-CIR and Self-RAG-CIR, 
which retrieve structured templates directly from the CIR knowledge base.
The results are shown in \cref{tab:performance-comparison}.

The results show that the R2C framework demonstrates substantial improvements, achieving 
47.2\% AR with deepseek-v3.2 compared to standard prompting (22.4\%) 
and outperforming methods such as Chain of Experts (30.2\%) 
and Self-RAG-raw (30.4\%). It also surpasses the proprietary LLM 
Grok4 in accuracy (47.2\% vs. 45.4\% AR). Furthermore, the framework shows consistent gains 
across model scales: while standard Qwen3-32B attains only 8\% AR, 
R2C with the same backbone reaches 21.2\% AR. 
These results highlight the effectiveness of our CIR-based 
framework in enabling systematic reasoning from complex 
operational rules to mathematical formulations.

\begin{table*}[h]
  \centering
  \caption{Sensitivity analysis of generation attempts (pass@k).}
  \label{tab:sensitivity}
  \begin{tabular}{lccccc}
    \toprule
    Model & Metric & pass@1 & pass@5 & pass@10 \\
    \midrule
    \multirow{2}{*}{DeepSeek-v3.2} & AR & 48.0\% & 64.0\% & 70.0\% \\
    & ER & 88.0\% & 92.0\% & 94.0\% \\
    \midrule
    \multirow{2}{*}{Qwen3-32B} & AR & 18.0\% & 22.4\% & 25.2\% \\
    & ER & 58.0\% & 76.0\% & 82.0\% \\
    \bottomrule
  \end{tabular}
\end{table*}

\subsection{Detailed Domain-wise Performance Analysis}
\label{sec:detailed-analysis}

\begin{table*}[h]
  \centering
  \caption{Accuracy Rate (AR) with reflection mechanism on multiple benchmarks (\%). Maximum 3 reflection iterations allowed.}
  \label{tab:reflection-results}
  \begin{tabular}{lcccc}
    \toprule
    \multirow{2}{*}{Method} & \multicolumn{4}{c}{Benchmark} \\
    \cmidrule{2-5}
    & IndustryOR & BWOR & OptMATH & ORCOpt-Bench \\
    \midrule
    R2C & 57.0 & 75.6 & 44.6 & 47.2 \\
    GPT-5 & 59.0 & 78.0 & 45.8 & 39.8 \\
    Grok-4 & \textbf{61.0} & \textbf{80.5} & 47.0 & 45.4 \\
    R2C with reflection & 59.0 & 79.2 & \textbf{51.8} & \textbf{54.0} \\
    \bottomrule
  \end{tabular}
\end{table*}

Building upon the performance comparison using the DeepSeek-v3.2-exp backbone, 
we analyze R2C's domain-specific capabilities. As shown in \cref{tab:domain-performance}, 
R2C matches or exceeds standard prompting performance across all domains, 
with no performance degradation. The framework substantially outperforms 
the baseline in most domains, particularly in Transportation 
(4 vs. 0 correct solutions), Resource Allocation (4 vs. 0), and Job Shop (2 vs. 0). 
These gains demonstrate R2C's effectiveness in handling complex rule interactions across diverse domains.

\textbf{Modeling multiplicity:} 
We observe that R2C framework can generate semantically equivalent but 
structurally distinct formulations compared to reference 
solutions in successful cases. In \autoref{app:instance-diversity}, 
we present several representative cases where R2C framework produced valid alternative 
modeling approaches. This modeling diversity demonstrates 
that the framework performs systematic modeling-strategy reasoning based on the knowledge of CIR templates
rather than rote memorization.

\subsection{Generalization to other benchmarks}
\label{sec:generalization}

To evaluate the generalizability of R2C framework to other
optimization problems, we test R2C on existing benchmarks: 
IndustryOR, BWOR, and OptMATH. \cref{tab:external-benchmarks} compares the AR 
performance of R2C against existing methods reported in the literature and 
our newly evaluated baselines. To ensure a fair comparison, we include 
methods only on the original benchmark instances without any modifications. 

The results show that R2C maintains strong performance across diverse benchmarks.
It achieves consistent improvement over standard prompting, 
demonstrating the transferable value of its structured, CIR-based reasoning.
For instance, using the DeepSeek-v3.2 backbone, R2C achieves 57.0\% AR on IndustryOR and 75.6\% AR on BWOR, 
closely trailing the proprietary model GPT-5 (59.0\% and 78.0\%, respectively).

We also observe limitations in the current framework. First, 
the knowledge corpus is bounded. The performance is contingent on the curated 
CIR library's domain coverage. Second, as discussed in \autoref{app:cir-formalisms}, 
while the CIR's soundness guarantee (\autoref{prop:cir-soundness-main}) ensures no rule-violating solutions, 
it currently does not achieve completeness. The CIR may instantiate overly restrictive 
constraints (e.g., applying a constraint to all indices, including irrelevant boundary cases), 
potentially excluding some valid solutions.
Future work should explore automating the construction and expansion of the CIR knowledge 
base to broaden its coverage, while also incorporating constraint refinement mechanisms into the framework.

\subsection{Ablation analysis}
\label{sec:ablation}

We conduct comprehensive ablation studies on ORCOpt-Bench to evaluate 
the contribution of each core component. We compare the full R2C framework 
against several ablated versions and a Base configuration 
(standard prompting with DeepSeek-v3.2, AR 22.4\%). 
Three ablation conditions are designed: (1) CIR Ablation (R2C w/o CIR), 
which replaces the structured CIR library with standard RAG over raw text; 
(2) Alternative Retrieval Ablations (Vanilla RAG-CIR and Self-RAG-CIR), 
which query the CIR library but use different retrieval strategies; and 
(3) Validation Ablation (R2C w/o Validation), which disables the 
Checker agent. All experiments use the DeepSeek-v3.2 backbone under consistent settings.

Results (\cref{tab:ablation}) show that each component contributes 
uniquely to the framework's performance. 
Specifically, removing the CIR causes the most substantial 
AR drop (from 47.2\% to 31.6\%), underscoring its fundamental 
role in systematic rule-to-constraint mapping. The retrieval 
variants (Vanilla/Self-RAG-CIR) perform better than the no-CIR 
baseline but worse than full R2C, indicating that the structured 
CIR retrieval is necessary.
Disabling validation leads to a moderate performance decrease 
(to 41.6\% AR), confirming its role in preventing error propagation through the modeling pipeline.

\subsection{Impact of Increased Generation Attempts}
\label{sec:impact-attempts}

We evaluate the Test-time scaling for the R2C framework by increasing 
the number of generation attempts (pass@k). \Cref{tab:sensitivity} presents the 
performance of two model backbones, DeepSeek-v3.2 and Qwen3-32B, on ORCOpt-Bench. 

The results demonstrate that both models benefit from additional attempts, 
with performance improving as k increases.
DeepSeek-v3.2 achieves 70.0\% AR at pass@10—an absolute gain of 22.0 percentage points 
from its pass@1 baseline (48.0\%)—while maintaining high execution robustness 
(ER improves from 88.0\% to 94.0\%). 
The smaller Qwen3-32B model exhibits a more modest but consistent improvement, 
with its AR increasing from 18.0\% at pass@1 to 25.2\% at pass@10.
These results demonstrate that R2C's performance can be further boosted through simple test-time repetition.

\subsection{Extension with Reflection Mechanism}
\label{sec:reflection-extension}

To demonstrate the extensibility of the proposed R2C framework, we further investigate extending our R2C framework with a reflection mechanism 
triggered by execution errors or solution infeasibility. 
We evaluate the reflection-augmented R2C framework on four benchmarks 
with a maximum of 3 reflection iterations per problem (see \autoref{app:reflection-extension} for prompt details). 
The results are shown in \cref{tab:reflection-results}.

As shown in \cref{tab:reflection-results}, incorporating reflection 
yields substantial performance improvements.
R2C attains 54.0\% AR on ORCOpt-Bench and 51.8\% on OptMATH, achieving 
state-of-the-art results.
This demonstrates that the structured, traceable CIR representation 
effectively supports iterative refinement.

\section*{Conclusion}

This work tackles the fundamental challenge of mapping complex natural-language 
operational rules to executable optimization constraints. We introduce 
the Canonical Intermediate Representation (CIR), a structured abstraction 
that decouples rule semantics from mathematical instantiation through 
constraint archetypes and paradigm candidates. Built upon CIR, the R2C framework 
implements a multi-agent pipeline that executes a structured three-stage workflow: 
it (1) parses problem descriptions into explicit operational rules; 
(2) retrieves and adapts relevant CIR templates, followed by paradigm-aware clustering; and 
(3) instantiates the mathematical model along with executable solver code.

Extensive evaluation demonstrates that R2C achieves state-of-the-art performance on ORCOpt-Bench, a benchmark we construct to 
evaluate complex operational rules and their interactions. On other optimization modeling benchmarks from prior work, 
R2C achieves new highly competitive results. 
Furthermore, when extended with a reflection mechanism, R2C sets new best-reported results on some benchmarks 
and remains highly competitive on the others, outperforming or matching strong baselines including proprietary LLMs such as GPT-5 and Grok-4.
A promising direction for future work is to automate the construction 
and expansion of the CIR knowledge base, enabling coverage of broader classes 
of optimization problems and domains.

\bibliography{main_arxiv}

\begin{thebibliography}{52}
\providecommand{\natexlab}[1]{#1}
\providecommand{\url}[1]{\texttt{#1}}
\expandafter\ifx\csname urlstyle\endcsname\relax
  \providecommand{\doi}[1]{doi: #1}\else
  \providecommand{\doi}{doi: \begingroup \urlstyle{rm}\Url}\fi

\bibitem[Ahmad et~al.(2021)Ahmad, Chakraborty, Ray, and
  Chang]{ahmad2021unified}
Ahmad, W.~U., Chakraborty, S., Ray, B., and Chang, K.-W.
\newblock Unified pre-training for program understanding and generation.
\newblock \emph{arXiv preprint arXiv:2103.06333}, 2021.

\bibitem[AhmadiTeshnizi et~al.(2024)AhmadiTeshnizi, Gao, and
  Udell]{ahmaditeshnizi2024optimus}
AhmadiTeshnizi, A., Gao, W., and Udell, M.
\newblock Optimus: Scalable optimization modeling with (mi) lp solvers and
  large language models.
\newblock \emph{arXiv preprint arXiv:2402.10172}, 2024.

\bibitem[Asai et~al.(2024)Asai, Wu, Wang, Sil, and Hajishirzi]{asai2024self}
Asai, A., Wu, Z., Wang, Y., Sil, A., and Hajishirzi, H.
\newblock Self-rag: Learning to retrieve, generate, and critique through
  self-reflection.
\newblock 2024.

\bibitem[Bazari et~al.(2023)Bazari, Pooya, Soleimani~Fard, and
  Roozkhosh]{bazari2023modeling}
Bazari, S., Pooya, A., Soleimani~Fard, O., and Roozkhosh, P.
\newblock Modeling and solving the problem of scheduling university exams in
  terms of new constraints on the conflicts of professors' exams and the
  concurrence of exams with common questions.
\newblock \emph{Opsearch}, 60\penalty0 (2):\penalty0 877--915, 2023.

\bibitem[Bertsimas \& Margaritis(2024)Bertsimas and
  Margaritis]{bertsimas2024robust}
Bertsimas, D. and Margaritis, G.
\newblock Robust and adaptive optimization under a large language model lens.
\newblock \emph{arXiv preprint arXiv:2501.00568}, 2024.

\bibitem[Cayirli \& Veral(2003)Cayirli and Veral]{cayirli2003outpatient}
Cayirli, T. and Veral, E.
\newblock Outpatient scheduling in health care: a review of literature.
\newblock \emph{Production and operations management}, 12\penalty0
  (4):\penalty0 519--549, 2003.

\bibitem[Chen et~al.(2025)Chen, Constante-Flores, Mantri, Kompalli, Ahluwalia,
  and Li]{chen2025optichat}
Chen, H., Constante-Flores, G.~E., Mantri, K. S.~I., Kompalli, S.~M.,
  Ahluwalia, A.~S., and Li, C.
\newblock Optichat: Bridging optimization models and practitioners with large
  language models.
\newblock \emph{arXiv preprint arXiv:2501.08406}, 2025.

\bibitem[Dauz{\`e}re-P{\'e}r{\`e}s et~al.(2024)Dauz{\`e}re-P{\'e}r{\`e}s, Ding,
  Shen, and Tamssaouet]{dauzere2024flexible}
Dauz{\`e}re-P{\'e}r{\`e}s, S., Ding, J., Shen, L., and Tamssaouet, K.
\newblock The flexible job shop scheduling problem: A review.
\newblock \emph{European Journal of Operational Research}, 314\penalty0
  (2):\penalty0 409--432, 2024.

\bibitem[Deblaere et~al.(2011)Deblaere, Demeulemeester, and
  Herroelen]{deblaere2011proactive}
Deblaere, F., Demeulemeester, E., and Herroelen, W.
\newblock Proactive policies for the stochastic resource-constrained project
  scheduling problem.
\newblock \emph{European Journal of Operational Research}, 214\penalty0
  (2):\penalty0 308--316, 2011.

\bibitem[Eszterg{\'a}r-Kiss \& Remeli(2021)Eszterg{\'a}r-Kiss and
  Remeli]{esztergar2021toward}
Eszterg{\'a}r-Kiss, D. and Remeli, V.
\newblock Toward practical algorithms for activity chain optimization.
\newblock \emph{Transportation Letters}, 13\penalty0 (1):\penalty0 64--76,
  2021.

\bibitem[Fan et~al.(2024)Fan, Ding, Ning, Wang, Li, Yin, Chua, and
  Li]{fan2024survey}
Fan, W., Ding, Y., Ning, L., Wang, S., Li, H., Yin, D., Chua, T.-S., and Li, Q.
\newblock A survey on rag meeting llms: Towards retrieval-augmented large
  language models.
\newblock In \emph{Proceedings of the 30th ACM SIGKDD conference on knowledge
  discovery and data mining}, pp.\  6491--6501, 2024.

\bibitem[F{\'e}vry et~al.(2020)F{\'e}vry, Soares, FitzGerald, Choi, and
  Kwiatkowski]{fevry2020entities}
F{\'e}vry, T., Soares, L.~B., FitzGerald, N., Choi, E., and Kwiatkowski, T.
\newblock Entities as experts: Sparse memory access with entity supervision.
\newblock \emph{arXiv preprint arXiv:2004.07202}, 2020.

\bibitem[Gao et~al.(2023)Gao, Xiong, Gao, Jia, Pan, Bi, Dai, Sun, Wang, and
  Wang]{gao2023retrieval}
Gao, Y., Xiong, Y., Gao, X., Jia, K., Pan, J., Bi, Y., Dai, Y., Sun, J., Wang,
  H., and Wang, H.
\newblock Retrieval-augmented generation for large language models: A survey.
\newblock \emph{arXiv preprint arXiv:2312.10997}, 2\penalty0 (1), 2023.

\bibitem[Green et~al.(2006)Green, Savin, and Wang]{green2006managing}
Green, L.~V., Savin, S., and Wang, B.
\newblock Managing patient service in a diagnostic medical facility.
\newblock \emph{Operations Research}, 54\penalty0 (1):\penalty0 11--25, 2006.

\bibitem[Heil et~al.(2020)Heil, Hoffmann, and Buscher]{heil2020railway}
Heil, J., Hoffmann, K., and Buscher, U.
\newblock Railway crew scheduling: Models, methods and applications.
\newblock \emph{European journal of operational research}, 283\penalty0
  (2):\penalty0 405--425, 2020.

\bibitem[Huang et~al.(2025)Huang, Tang, Hu, Jiang, Zheng, Ge, Wang, and
  Wang]{huang2025orlm}
Huang, C., Tang, Z., Hu, S., Jiang, R., Zheng, X., Ge, D., Wang, B., and Wang,
  Z.
\newblock Orlm: A customizable framework in training large models for automated
  optimization modeling.
\newblock \emph{Operations Research}, 2025.

\bibitem[Iwamura \& Sugimura(2010)Iwamura and Sugimura]{iwamura2010study}
Iwamura, K. and Sugimura, N.
\newblock A study on real-time scheduling for autonomous distributed
  manufacturing systems.
\newblock In \emph{2010 IEEE International Conference on Systems, Man and
  Cybernetics}, pp.\  1352--1357. IEEE, 2010.

\bibitem[Jans \& Degraeve(2008)Jans and Degraeve]{jans2008modeling}
Jans, R. and Degraeve, Z.
\newblock Modeling industrial lot sizing problems: a review.
\newblock \emph{International Journal of Production Research}, 46\penalty0
  (6):\penalty0 1619--1643, 2008.

\bibitem[Jiang et~al.(2024)Jiang, Shu, Qian, Lu, Zhou, Zhou, and
  Yu]{jiang2024llmopt}
Jiang, C., Shu, X., Qian, H., Lu, X., Zhou, J., Zhou, A., and Yu, Y.
\newblock Llmopt: Learning to define and solve general optimization problems
  from scratch.
\newblock \emph{arXiv preprint arXiv:2410.13213}, 2024.

\bibitem[Jiang et~al.(2025)Jiang, Wu, Zhang, and Zhang]{jiang2025droc}
Jiang, X., Wu, Y., Zhang, C., and Zhang, Y.
\newblock Droc: Elevating large language models for complex vehicle routing via
  decomposed retrieval of constraints.
\newblock In \emph{13th international Conference on Learning Representations,
  ICLR 2025}. OpenReview. net, 2025.

\bibitem[Joshi et~al.(2023)Joshi, Sanchez, Gulwani, Le, Verbruggen, and
  Radi{\v{c}}ek]{joshi2023repair}
Joshi, H., Sanchez, J.~C., Gulwani, S., Le, V., Verbruggen, G., and
  Radi{\v{c}}ek, I.
\newblock Repair is nearly generation: Multilingual program repair with llms.
\newblock In \emph{Proceedings of the AAAI Conference on Artificial
  Intelligence}, volume~37, pp.\  5131--5140, 2023.

\bibitem[Levonian et~al.(2023)Levonian, Li, Zhu, Gade, Henkel, Postle, and
  Xing]{levonian2023retrieval}
Levonian, Z., Li, C., Zhu, W., Gade, A., Henkel, O., Postle, M.-E., and Xing,
  W.
\newblock Retrieval-augmented generation to improve math question-answering:
  Trade-offs between groundedness and human preference.
\newblock \emph{arXiv preprint arXiv:2310.03184}, 2023.

\bibitem[Li et~al.(2023)Li, Mellou, Zhang, Pathuri, and Menache]{li2023large}
Li, B., Mellou, K., Zhang, B., Pathuri, J., and Menache, I.
\newblock Large language models for supply chain optimization.
\newblock \emph{arXiv preprint arXiv:2307.03875}, 2023.

\bibitem[Liang et~al.(2025)Liang, Lu, Mao, Sun, Zeng, Jin, Qin, Zhu, Teo,
  et~al.]{liang2025llm}
Liang, K., Lu, Y., Mao, J., Sun, S., Zeng, C., Jin, X., Qin, H., Zhu, R., Teo,
  C.-P., et~al.
\newblock Llm for large-scale optimization model auto-formulation: A
  lightweight few-shot learning approach.
\newblock \emph{Congcong and Jin, Xiao and Qin, Hanzhang and Zhu, Ruihao and
  Teo, Chung-Piaw, LLM for Large-Scale Optimization Model Auto-Formulation: A
  Lightweight Few-Shot Learning Approach (June 28, 2025)}, 2025.

\bibitem[Liu et~al.(2024)Liu, Tong, Yuan, Lin, Luo, Wang, Lu, and
  Zhang]{liu2024evolution}
Liu, F., Tong, X., Yuan, M., Lin, X., Luo, F., Wang, Z., Lu, Z., and Zhang, Q.
\newblock Evolution of heuristics: Towards efficient automatic algorithm design
  using large language model.
\newblock \emph{arXiv preprint arXiv:2401.02051}, 2024.

\bibitem[Lu et~al.(2025)Lu, Xie, Wu, Ren, Chen, and Wen]{lu2025optmath}
Lu, H., Xie, Z., Wu, Y., Ren, C., Chen, Y., and Wen, Z.
\newblock Optmath: A scalable bidirectional data synthesis framework for
  optimization modeling.
\newblock \emph{arXiv preprint arXiv:2502.11102}, 2025.

\bibitem[Lu et~al.(2022)Lu, Duan, Han, Guo, Hwang, and
  Svyatkovskiy]{lu2022reacc}
Lu, S., Duan, N., Han, H., Guo, D., Hwang, S.-w., and Svyatkovskiy, A.
\newblock Reacc: A retrieval-augmented code completion framework.
\newblock \emph{arXiv preprint arXiv:2203.07722}, 2022.

\bibitem[Maenhout \& Vanhoucke(2010)Maenhout and Vanhoucke]{maenhout2010hybrid}
Maenhout, B. and Vanhoucke, M.
\newblock A hybrid scatter search heuristic for personalized crew rostering in
  the airline industry.
\newblock \emph{European journal of operational research}, 206\penalty0
  (1):\penalty0 155--167, 2010.

\bibitem[Merkert et~al.(2015)Merkert, Harjunkoski, Isaksson, S{\"a}ynevirta,
  Saarela, and Sand]{merkert2015scheduling}
Merkert, L., Harjunkoski, I., Isaksson, A., S{\"a}ynevirta, S., Saarela, A.,
  and Sand, G.
\newblock Scheduling and energy--industrial challenges and opportunities.
\newblock \emph{Computers \& Chemical Engineering}, 72:\penalty0 183--198,
  2015.

\bibitem[Mostajabdaveh et~al.(2024)Mostajabdaveh, Yu, Ramamonjison, Carenini,
  Zhou, and Zhang]{mostajabdaveh2024optimization}
Mostajabdaveh, M., Yu, T.~T., Ramamonjison, R., Carenini, G., Zhou, Z., and
  Zhang, Y.
\newblock Optimization modeling and verification from problem specifications
  using a multi-agent multi-stage llm framework.
\newblock \emph{INFOR: Information Systems and Operational Research},
  62\penalty0 (4):\penalty0 599--617, 2024.

\bibitem[OpenAI(2025)]{openai2025gpt5}
OpenAI.
\newblock Introducing gpt-5.
\newblock \url{https://openai.com}, August 2025.

\bibitem[Parvez et~al.(2021)Parvez, Ahmad, Chakraborty, Ray, and
  Chang]{parvez2021retrieval}
Parvez, M.~R., Ahmad, W.~U., Chakraborty, S., Ray, B., and Chang, K.-W.
\newblock Retrieval augmented code generation and summarization.
\newblock \emph{arXiv preprint arXiv:2108.11601}, 2021.

\bibitem[Perumal et~al.(2022)Perumal, Lusby, and Larsen]{perumal2022electric}
Perumal, S.~S., Lusby, R.~M., and Larsen, J.
\newblock Electric bus planning \& scheduling: A review of related problems and
  methodologies.
\newblock \emph{European Journal of Operational Research}, 301\penalty0
  (2):\penalty0 395--413, 2022.

\bibitem[Pillay(2016)]{pillay2016review}
Pillay, N.
\newblock A review of hyper-heuristics for educational timetabling.
\newblock \emph{Annals of Operations Research}, 239\penalty0 (1):\penalty0
  3--38, 2016.

\bibitem[Ramamonjison et~al.(2022)Ramamonjison, Li, Yu, He, Rengan,
  Banitalebi-Dehkordi, Zhou, and Zhang]{ramamonjison2022augmenting}
Ramamonjison, R., Li, H., Yu, T., He, S., Rengan, V., Banitalebi-Dehkordi, A.,
  Zhou, Z., and Zhang, Y.
\newblock Augmenting operations research with auto-formulation of optimization
  models from problem descriptions.
\newblock In \emph{Proceedings of the 2022 Conference on Empirical Methods in
  Natural Language Processing: Industry Track}, pp.\  29--62, 2022.

\bibitem[Ribeiro(2012)]{ribeiro2012sports}
Ribeiro, C.~C.
\newblock Sports scheduling: Problems and applications.
\newblock \emph{International Transactions in Operational Research},
  19\penalty0 (1-2):\penalty0 201--226, 2012.

\bibitem[Rubin et~al.(2022)Rubin, Herzig, and Berant]{rubin2022learning}
Rubin, O., Herzig, J., and Berant, J.
\newblock Learning to retrieve prompts for in-context learning.
\newblock In \emph{Proceedings of the 2022 conference of the North American
  chapter of the association for computational linguistics: human language
  technologies}, pp.\  2655--2671, 2022.

\bibitem[Sahling et~al.(2009)Sahling, Buschk{\"u}hl, Tempelmeier, and
  Helber]{sahling2009solving}
Sahling, F., Buschk{\"u}hl, L., Tempelmeier, H., and Helber, S.
\newblock Solving a multi-level capacitated lot sizing problem with
  multi-period setup carry-over via a fix-and-optimize heuristic.
\newblock \emph{Computers \& Operations Research}, 36\penalty0 (9):\penalty0
  2546--2553, 2009.

\bibitem[S{\'a}nchez et~al.(2023)S{\'a}nchez, Lalla-Ruiz, Gil, Castro, and
  Vo{\ss}]{sanchez2023resource}
S{\'a}nchez, M.~G., Lalla-Ruiz, E., Gil, A.~F., Castro, C., and Vo{\ss}, S.
\newblock Resource-constrained multi-project scheduling problem: A survey.
\newblock \emph{European Journal of Operational Research}, 309\penalty0
  (3):\penalty0 958--976, 2023.

\bibitem[Shinn et~al.(2023)Shinn, Cassano, Gopinath, Narasimhan, and
  Yao]{shinn2023reflexion}
Shinn, N., Cassano, F., Gopinath, A., Narasimhan, K., and Yao, S.
\newblock Reflexion: Language agents with verbal reinforcement learning.
\newblock \emph{Advances in Neural Information Processing Systems},
  36:\penalty0 8634--8652, 2023.

\bibitem[Sun et~al.(2024)Sun, Ye, Zhang, Huang, Zhang, Wei, and
  Cai]{sun2024autosat}
Sun, Y., Ye, F., Zhang, X., Huang, S., Zhang, B., Wei, K., and Cai, S.
\newblock Autosat: Automatically optimize sat solvers via large language
  models.
\newblock \emph{arXiv preprint arXiv:2402.10705}, 2024.

\bibitem[T{\"u}ys{\"u}z et~al.(2024)T{\"u}ys{\"u}z, Okumu{\c{s}}, Aymaz, and
  {\c{C}}avdar]{tuysuz2024real}
T{\"u}ys{\"u}z, M., Okumu{\c{s}}, H.~I., Aymaz, {\c{S}}., and {\c{C}}avdar, B.
\newblock Real-time application of a demand-side management strategy using
  optimization algorithms.
\newblock \emph{Applied Energy}, 368:\penalty0 123373, 2024.

\bibitem[{x.ai}(2025)]{xai2025grok4}
{x.ai}.
\newblock Grok 4.
\newblock \url{https://x.ai/news/grok-4}, 2025.

\bibitem[Xiao et~al.(2023)Xiao, Zhang, Wu, Xu, Wang, Han, Fu, Zhong, Zeng,
  Song, et~al.]{xiao2023chain}
Xiao, Z., Zhang, D., Wu, Y., Xu, L., Wang, Y.~J., Han, X., Fu, X., Zhong, T.,
  Zeng, J., Song, M., et~al.
\newblock Chain-of-experts: When llms meet complex operations research
  problems.
\newblock In \emph{The twelfth international conference on learning
  representations}, 2023.

\bibitem[Yang et~al.(2023)Yang, Swope, Gu, Chalamala, Song, Yu, Godil, Prenger,
  and Anandkumar]{yang2023leandojo}
Yang, K., Swope, A., Gu, A., Chalamala, R., Song, P., Yu, S., Godil, S.,
  Prenger, R.~J., and Anandkumar, A.
\newblock Leandojo: Theorem proving with retrieval-augmented language models.
\newblock \emph{Advances in Neural Information Processing Systems},
  36:\penalty0 21573--21612, 2023.

\bibitem[Ye et~al.(2024)Ye, Wang, Cao, Berto, Hua, Kim, Park, and
  Song]{ye2024reevo}
Ye, H., Wang, J., Cao, Z., Berto, F., Hua, C., Kim, H., Park, J., and Song, G.
\newblock Reevo: Large language models as hyper-heuristics with reflective
  evolution.
\newblock \emph{Advances in neural information processing systems},
  37:\penalty0 43571--43608, 2024.

\bibitem[Zhan et~al.(2015)Zhan, Liu, Gong, Zhang, Chung, and Li]{zhan2015cloud}
Zhan, Z.-H., Liu, X.-F., Gong, Y.-J., Zhang, J., Chung, H. S.-H., and Li, Y.
\newblock Cloud computing resource scheduling and a survey of its evolutionary
  approaches.
\newblock \emph{ACM Computing Surveys (CSUR)}, 47\penalty0 (4):\penalty0 1--33,
  2015.

\bibitem[Zhang \& Luo(2025)Zhang and Luo]{zhang2025or}
Zhang, B. and Luo, P.
\newblock Or-llm-agent: Automating modeling and solving of operations research
  optimization problem with reasoning large language model.
\newblock \emph{arXiv preprint arXiv:2503.10009}, 2025.

\bibitem[Zhang et~al.(2025)Zhang, Shi, and Tang]{zhang2025optimization}
Zhang, J., Shi, X., and Tang, Y.
\newblock Optimization of sports event operations through algorithmic
  scheduling and management.
\newblock \emph{International Journal of Information System Modeling and Design
  (IJISMD)}, 16\penalty0 (1):\penalty0 1--23, 2025.

\bibitem[Zhao et~al.(2024)Zhao, Zhang, Yu, Wang, Geng, Fu, Yang, Zhang, Jiang,
  and Cui]{zhao2024retrieval}
Zhao, P., Zhang, H., Yu, Q., Wang, Z., Geng, Y., Fu, F., Yang, L., Zhang, W.,
  Jiang, J., and Cui, B.
\newblock Retrieval-augmented generation for ai-generated content: A survey.
\newblock \emph{arXiv preprint arXiv:2402.19473}, 2024.

\bibitem[Zhou et~al.(2025)Zhou, Yang, Xin, Chen, He, and Ge]{zhou2025auto}
Zhou, C., Yang, J., Xin, L., Chen, Y., He, Z., and Ge, D.
\newblock Auto-formulating dynamic programming problems with large language
  models.
\newblock \emph{arXiv preprint arXiv:2507.11737}, 2025.

\bibitem[Zhou et~al.(2023)Zhou, Fan, Zhou, Li, Lei, Xu, and
  Nallanathan]{zhou2023priority}
Zhou, W., Fan, L., Zhou, F., Li, F., Lei, X., Xu, W., and Nallanathan, A.
\newblock Priority-aware resource scheduling for uav-mounted mobile edge
  computing networks.
\newblock \emph{IEEE Transactions on Vehicular Technology}, 72\penalty0
  (7):\penalty0 9682--9687, 2023.

\end{thebibliography}
\bibliographystyle{arxiv2026}
%%%%%%%%%%%%%%%%%%%%%%%%%%%%%%%%%%%%%%%%%%%%%%%%%%%%%%%%%%%%%%%%%%%%%%%%%%%%%%%
%%%%%%%%%%%%%%%%%%%%%%%%%%%%%%%%%%%%%%%%%%%%%%%%%%%%%%%%%%%%%%%%%%%%%%%%%%%%%%%
% APPENDIX
%%%%%%%%%%%%%%%%%%%%%%%%%%%%%%%%%%%%%%%%%%%%%%%%%%%%%%%%%%%%%%%%%%%%%%%%%%%%%%%
%%%%%%%%%%%%%%%%%%%%%%%%%%%%%%%%%%%%%%%%%%%%%%%%%%%%%%%%%%%%%%%%%%%%%%%%%%%%%%%
\newpage
\appendix
\onecolumn

\section*{Appendix}
\addcontentsline{toc}{section}{Appendix}

\subsection*{Appendix Contents}
\begin{itemize}
  \item \ref{app:cir-formalisms} -- Formal Semantics and Soundness of the CIR
  \item \ref{app:prompts} -- Prompts of R2C agents
  \begin{itemize}
    \item \ref{app:extractor} -- The Extractor Agent
    \item \ref{app:mapper} -- The Mapper Agent
    \item \ref{app:formalizer} -- The Formalizer Agent
    \item \ref{app:checker-prompts} -- Checker Agent Verification Prompts
    \item \ref{app:standard-prompt} -- Standard Prompting Baseline
  \end{itemize}
  \item \ref{app:cir} -- CIR domain knowledge construction
  \begin{itemize}
    \item \ref{app:cir:examples} -- CIR Examples
  \end{itemize}
  \item \ref{app:orcoopt} -- ORCOpt-Bench: Operation Rule to Constraint Optimization Benchmark
  \item \ref{app:instance-diversity} -- Instance code diversity
  \item \ref{app:reflection-extension} -- Extension with Reflection Mechanism
\end{itemize}

\section{Formal Semantics and Soundness of the CIR}
\label{app:cir-formalisms}

This appendix provides the formal definitions and proofs underlying the Canonical Intermediate Representation (CIR) and its soundness guarantee stated in Proposition~\ref{prop:cir-soundness-main} of the main text.

\subsection{Formal Semantic Foundations}
\label{app:cir-semantics}

Let $\mathcal{X}$ be the space of possible decisions for a problem. A solution $x \in \mathcal{X}$ satisfies a rule $r_k$, denoted $x \models r_k$, if it adheres to the rule's logical intent. The set of all solutions satisfying every rule in $\mathcal{R}(d)$ is:
\[
\mathcal{F}_{\text{rules}}(d) = \{x \in \mathcal{X} \mid \forall r_k \in \mathcal{R}(d), x \models r_k\}.
\]

\paragraph{Semantic Intent of a CIR Template.}
Each CIR template corresponds to a core operational intent $A$, 
characterized by a logical formula $\varphi_A(x)$. 
This formula defines the intended behavior of any solution $x$ with respect to that intent, 
independently of any modeling paradigm. 
For instance, $\varphi_{\text{SystemPowerBalance}}(x)$ states that ``power supply matches 
demand at all times,'' abstracting away time representation (e.g., as a continuous variable or a discrete index).

\paragraph{From Rules to CIR Semantics.}
When a rule $r_k$ is mapped to CIR implementations $\{a_\ell = (A_\ell, k_\ell, p_\ell) \mid k_\ell = k\}$, the CIR-semantics of $r_k$ is defined as the logical conjunction of the semantic intents of all its constituent parts:
\begin{equation}
\varphi_{r_k}^{\text{CIR}}(x) := \bigwedge_{\ell: k_\ell = k} \varphi_{A_\ell}(x).
\label{eq:cir-rule-semantics}
\end{equation}

\paragraph{Design Guarantee (Semantic Equivalence).}
The CIR knowledge base is constructed to ensure that for any well-formed rule $r_k$, its CIR-based representation is semantically equivalent to the original rule:
\begin{equation}
x \models r_k \quad \text{if and only if} \quad \varphi_{r_k}^{\text{CIR}}(x) \text{ holds}.
\label{eq:design-equivalence}
\end{equation}

\subsection{Soundness Proposition and Proof}
\label{app:proof-soundness}
\begin{proposition}[Full CIR Soundness Guarantee]
\label{prop:cir-soundness-full}
Assume the CIR knowledge base satisfies:
\begin{enumerate}
	\item \textbf{Semantic Design:} For each rule \(r_k\), if a CIR implementation for \(r_k\) is 
  generated by instantiating a template from the knowledge base, then its 
  CIR-semantics \(\varphi_{r_k}^{\text{CIR}}(x)\) is semantically equivalent to the original rule.
	\item \textbf{Sound Instantiation:} For every intent $A$ and paradigm $p$, the process 
  of instantiating the corresponding template to produce concrete constraints 
  $\mathcal{C}_{A,p}$ is sound. Formally, any solution satisfying the concrete 
  constraints must satisfy the intent's semantic formula:
	\[
	\text{If } x \text{ satisfies } \mathcal{C}_{A,p}, \text{ then } \varphi_A(x) \text{ holds.}
	\]
\end{enumerate}
Then, for any problem description $d$ and its CIR $C(d)$, the translated model $M(d) = \bigcup_{\ell=1}^{L} \mathcal{C}_{A_\ell, p_\ell}$ is sound with respect to the original rules:
\[
\mathcal{F}(M(d)) \subseteq \mathcal{F}_{\text{rules}}(d).
\]
\textbf{Scope:} This guarantee applies to any CIR representation \(C(d)\) where the implementations in \(\mathcal{A}(d)\) are generated by instantiating templates from the CIR knowledge base.
\end{proposition}

\begin{proof}
Let $\mathcal{X}$ be the decision space. A solution $x \in \mathcal{X}$ is rule-consistent if it satisfies every rule $r_k \in \mathcal{R}(d)$, denoted by $x \models r_k$. The set of all rule-consistent solutions is:
\[
\mathcal{F}_{\text{rules}}(d) = \{x \in \mathcal{X} : x \models r_k \text{ for all } r_k \in \mathcal{R}(d)\}.
\]

We prove that any $x \in \mathcal{F}(M(d))$ necessarily belongs to $\mathcal{F}_{\text{rules}}(d)$.
First, by definition of $M(d)$ (Equation~\eqref{eq:CIR-to-model} in the main text), $x \in \mathcal{F}(M(d))$ implies that $x$ satisfies every concrete constraint block $\mathcal{C}_{A_\ell, p_\ell}$ for all $\ell = 1, \ldots, L$.
Then, by the Sound Instantiation assumption (2), for each $\ell$, since $x$ satisfies $\mathcal{C}_{A_\ell, p_\ell}$, it must satisfy the semantic intent of the corresponding archetype: $\varphi_{A_\ell}(x)$ holds.
Next, consider any rule $r_k \in \mathcal{R}(d)$. Let $\{\ell \mid k_\ell = k\}$ be the 
indices of its associated CIR implementations. From the previous step, $x$ satisfies $\varphi_{A_\ell}(x)$ for all $\ell$ in this set. Therefore, by the definition of CIR-semantics (Equation~\eqref{eq:cir-rule-semantics}), $x$ satisfies their logical conjunction $\varphi_{r_k}^{\text{CIR}}(x)$.
Finally, by the Semantic Design guarantee (1), satisfying $\varphi_{r_k}^{\text{CIR}}(x)$ is equivalent to satisfying the original rule $r_k$. Hence, $x \models r_k$.
Since $r_k$ was arbitrary, $x$ satisfies all rules in $\mathcal{R}(d)$. Therefore, $x \in \mathcal{F}_{\text{rules}}(d)$, completing the proof.
\end{proof}
The guarantee provided by Proposition~\ref{prop:cir-soundness-full} is one of 
soundness (the model does not permit rule-violating solutions), not necessarily 
completeness (the model might exclude some rule-consistent solutions). 
Completeness depends on the specific constraints generated by the 
paradigm-specific translators. The framework's primary contribution is 
providing a structured pathway to achieve soundness, a critical foundation for reliable automated modeling.

\section{Prompts of R2C agents}
\label{app:prompts}

\subsection{The Extractor Agent}
\label{app:extractor}

The Extractor Agent parses unstructured natural language problem descriptions into structured representations. 
Below we present the complete system prompt used for the Extractor Agent:

\begin{quote}
\begin{Verbatim}[fontsize=\footnotesize,breaklines=true,breaksymbol={},samepage=false]
  You are the "Universal Operations Problem Extractor".

  Mission
  
  * Read a natural-language description of an optimization problem.
  * Understand the domain and intent.
  * Output ONE strict JSON object (no extra text, no Markdown fences) that captures: entities, measures/metrics, explicit rules, implicit signals, and the target cost/objective**.
  * The objective must exactly mirror the problem's stated ask; do not introduce surrogate goals, re-weighting, or invented penalties.
  * Define the objective so that the optimizer's final value/selection is directly the requested answer (no post-hoc conversions).
  
  Domains & Tagging
  
  * Allowed domain tags (choose one or more):    
  ["crew","education","energy","healthcare","supply_chain_and_production",
  "project","resource","sports","job shop","transportation"].
  
  Unify Units
  
  * Unify all units in the problem. Through reasonable calculation, some value will be converted into a unified unit.
  * Currencies: keep symbol if present, else use a "currency" field and numeric value.
  * Distances, loads, power, inventory, etc., must include units where known.
  
  Completeness & Traceability
  
  * **All numeric quantities in prose** must **appear explicitly in the JSON** (never invent values).
  * Carry prose numbers into the correct section: `entities`, `measures`, `explicit_rules`, `implicit_signals`, `objective`, or `given_parameters.constants`.
  * **All table-derived numbers** must live in the single top-level `"tabular_values"` field; elsewhere you may reference conceptually, but don't duplicate raw values.
  * Output must be **minified**, preserve defined **key order**, and keep **arrays** even when empty.
  * objective must be a faithful paraphrase (or direct formalization) of the statement's explicit goal; if multiple is mentioned,You should mark two mainstream scalarization paradigms: 
  * Weighted-sum aggregation: merge objectives into a single scalar via convex weights.
  * Lexicographic (pre-emptive) optimization: strictly prioritize objectives and solve a hierarchical sequence of single-objective problems.
  
  Rules
  
  * When extracting numbers, keep units and any bounds(<=, >=); preserve original text in "explicit\_rules[*].text".
  * If both explicit and implicit constraints exist for the same idea. **just include explicit**
  * Do not alter the optimization target into a proxy metric or add regularizes not present in the statement. Unless the problem description clarify.
  * No invented weights/coefficients; only use weights if explicitly provided, otherwise omit them.
  * If the statement asks to "minimize/maximize …", ensure objective reproduces that exact quantity, and that the optimizer’s returned value/list is the final answer without additional mapping.
  

  Following is the entry:
  {Entry}
  

\end{Verbatim}
\end{quote}

\textbf{Output Schema:} The Extractor Agent produces a JSON object adhering to the following structure:

\begin{quoting}
  \begin{Verbatim}[fontsize=\small,breaklines=true,breaksymbol={},samepage=false]
      "domain_tags": "...",
      "problem_summary": "One-sentence paraphrase of the user's ask and scope.",
      "entities": {
        "IMPORTANT NOTE": "BELOW IS AN EXAMPLE. YOU SHOULD EXTRACT FROM THE STATEMENT ",
        "tasks_jobs": {
          "definition": "Work items to schedule/route/allocate.",
          "records": [
            {
              "id": "string",
              "name": "string|none",
              "start_time": "timestamp|relative|none",
              "end_time": "timestamp|relative|duration|none",
              "duration": "duration|none",
              "location_or_stage": "enum|string|none",
              "demand_or_load": "scalar|vector|none",
              "priority": "high|medium|low|numeric|none",
              "precedence": ["task_id", "..."]
            }
          ]
        },
        "resources": {
          "types": [
            {
              "name": "human|machine|vehicle|room_bed|inventory|energy|other as statement clarify",
              "attributes": ["skills", "shift_limit", "rest_rules", "capacity", "setup", "maintenance_windows", "depot", "range", "driver_link", "availability", "balance", "ramp_rate", "fairness"],
              "counts_limits": "describe available quantity, pools, or bounds if given"
            }
          ]
        },
        "background_loads": {
          "definition": "Time-varying but predictable demands that must be accounted for in resource constraints",
          "records": [
            {
              "id": "string",
              "name": "string|none",
              "load_type": "constant|cyclic|time_series|stochastic",
              "profile": {
                "pattern": "string|none",
                "base_value": "numeric|none",
                "variation": "numeric|none",
                "period": "duration|none",
                "duty_cycle": "ratio|none"
              },
              "time_window": {
                "start": "timestamp|relative|none",
                "end": "timestamp|relative|none"
              }
            }
          ]
        }
        "locations_nodes": {
          "space": [],
          "travel_or_transition": {
            "has_travel": true|false,
            "metric": "time|distance|cost|none"
          }
        }
      },
      "time_model": {
        "scale": "continuous|discrete",
        "granularity": "minutes|hours|slots|none",
        "horizon": {
          "start": "timestamp|relative|none",
          "end": "timestamp|relative|none"
        }
      },
      "entity_count_estimates": {
        "tasks": "int|none",
        "resources": "int|none",
        "locations": "int|none"
      },
      "measures": [
        {
          "name": "e.g., cost, lateness, utilization, on_time_rate, energy_use, distance, overtime",
          "value": "numeric|none",
          "unit": "currency|min|kWh|km|units|%|none",
          "scope": "task|resource|route|system",
          "hint": "how it is computed if stated"
        }
      ],
      "explicit_rules": [
        {
          "rid": "R1",
          "text": "Original sentence fragment from the user.",
          "type": "assignment|coverage|time_window|sequence|precedence|setup|travel|idle|break|capacity|balance|ramp|fairness|compatibility|inventory|energy|depot|return|multi_period|others",
          "applies_to": ["task_id|resource_id|type|all"],
          "hard": true
        },
        {
          "rid": "D1",
          "text": "Decision variables are well-defined and integral where stated.",
          "type": "domain",
          "applies_to": ["all"],
          "hard": true
        },
        {
          "rid": "L1",
          "text": "Each execution of an action corresponds to one trip/unit use.",
          "type": "assignment",
          "applies_to": ["all"],
          "hard": true
        },
        {
          "rid": "C1",
          "text": "Resource capacities must be respected.",
          "type": "capacity",
          "applies_to": ["all"],
          "hard": true
        },
        {
          "rid": "U1",
          "text": "Per-execution fixed and rate-based costs/consumptions apply.",
          "type": "energy",
          "applies_to": ["all"],
          "hard": true
        },
        {
          "rid": "P1",
          "text": "Penalty costs for various constraint violations as specified in objective",
          "type": "penalty_cost",
          "applies_to": ["all"],
          "hard": false
        }
    
      ],
      "implicit_signals": [
        "Inferred assumptions or likely missing constraints (e.g., resource scarcity, fairness rotation, vehicle must return to depot, rest times implied by labor law, ramp limits for generators, inventory cannot be negative, rooms have capacity 1, if not specific the machine is down at the beginning, the tank is empty initially etc.).",
        "Within each task, execution outcomes are independent across repetitions (if stated).",
        "Outcomes across different tasks are independent (if stated).",
        "Each execution covers exactly one unit of work for its task and action."
      ],
      "given_parameters": {
        "matrices_or_functions": [
          "travel_time(u,v)",
          "setup_time(i,j)",
          "capacity(r,t)",
          "energy_cost(t)",
          "base_success_prob(i,a)",
          "task_success_prob(i, x_i) = 1 - \\Pi_{a}(1 - base_success_prob(i,a))^{x_{i,a}}",
          "system_success_prob(x) = P(\\Sigma_i Z_i \\geq k) with Z_i ~ Bernoulli(task_success_prob(i,x_i))",
          "consumption(i,a,r) = fixed[a,r] + \\Sigma_m rate[a,r,m]\\cdot metric[i,m]"
        ],
        "tables_or_data": [
          "task_table",
          "resource_table",
          "shift_calendar",
          "distance_matrix",
          "success_prob_table",
          "resource_cost_table"
        ],
        "constants": {
          "K": 500,
          "cost_per_min": 10,
          "max_overtime_hr": 2
        }
      },
      "objective": {
        "goal": "min_cost|max_throughput|max_utility|min_lateness|multiobjective| (as statement clarify. Strictly restate the optimization objective in the question (e.g., \"maximize P(Sum_i Z_i \\geq k)\",'minimize the total delay in minutes', 'maximize the net profit').)",
        "target_cost": "numeric|expression|none",
        "components": [
          {
            "name": "activation_cost",
            "sign": "+",
            "weight": 1.0,
            "applies_to": "resource"
          },
          {
            "name": "travel_cost",
            "sign": "+",
            "weight": 1.0,
            "applies_to": "edge"
          },
          {
            "name": "overtime_penalty",
            "sign": "+",
            "weight": 1.0,
            "applies_to": "resource_time"
          },
          {
            "name": "lateness_penalty",
            "sign": "+",
            "weight": 1.0,
            "applies_to": "task"
          },
          {
            "name": "revenue_or_utility",
            "sign": "-",
            "weight": 1.0,
            "applies_to": "task|assignment"
          }
        ]
      },
      "tabular_values": [
        {
          "name": "string (table title or inferred)",
          "schema": ["col1", "col2", "..."],
          "units": {
            "col_name": "unit|none"
          },
          "rows": [
            ["compact row values, strings or numbers in order"],
            ["..."]
          ],
          "notes": "optional short note or empty string"
        }
      ],
      "domain_examples": {
        "IMPORTANT NOTE": "THIS FIELD SHOULD NOT BE INCLUDED IN THE OUTPUT",
        "crew": [
          "Each flight requires 2 pilots and 3 attendants",
          "Crew rest \\geq 12h between duties"
        ],
        "education": [
          "A teacher cannot teach two classes simultaneously",
          "Room capacity cannot be exceeded"
        ],
        "energy": [
          "Ramp rate <= 20 MW/min",
          "Unit minimum up/down times apply"
        ],
        "healthcare": [
          "Surgeon needs 30 min turnover between cases",
          "ICU beds have capacity 1"
        ],
        "supply_chain_and_production": [
          "Demand must be met weekly",
          "Inventory balance: inv(t)=inv(t-1)+inbound-outbound"
        ],
        "project": [
          "Task B depends on Task A completion",
          "Finish by milestone date"
        ],
        "resource": [
          "Do not double-book a resource",
          "Respect skill-to-task compatibility"
        ],
        "sports": [
          "No player plays two matches at the same time",
          "Home/away balance over season"
        ],
        "job shop": [
          "Each machine processes one job at a time",
          "Operations follow fixed routing sequence"
        ],
        "transportation": [
          "Vehicle starts/ends at depot",
          "Route duration <= driver's shift"
        ]
      }
  
  \end{Verbatim}
\end{quoting}

\subsection{The Mapper Agent}
\label{app:mapper}

The Mapper Agent transforms structured problem elements from the Extractor into formal modeling representations. 
It performs constraint template mapping, CIR clustering, paradigm selection, based on domain knowledge.
Below we present the complete system prompt used for the Mapper Agent:

\begin{quote}
\begin{Verbatim}[fontsize=\small,breaklines=true,breaksymbol={},samepage=false]
  You are an excellent expert and a top researcher of Operation optimization problems. You are responde to generate **Modeling Mapper** for **{domain_tag}** scheduling, emit concise rationale fields (why/assumptions/limits) inside the JSON rather than free-form chain-of-thought; use verifiable, model-side justifications only. Read the **Extraction** and emit a **single Mapper JSON** that maps natural-language operation rules to: (a) a rich **catalog of constraint candidate modeling paradigms for operation rules **, (b) **sets and variables**, (c) optional **graph/time indexing recipes**, and (d) an explicit **cluster relationship** among operation rules.
  The mapper must (i) exactly echo IDs/units/time granularity from Extraction; (ii) list any missing fields as data_gaps with placeholders and bounds.  
  
  ## Original_Problem
  {Original_Problem}
  
  ## Extraction
  {Extraction}
  
  ## Domain Knowledge Infusion
  
  The domain knowledge supplies a sufficiently comprehensive inventory of candidate paradigms, detailing name, applicability, advantages, and risks; the names in top_paradigms must be selected from the domain knowledge. If there is no suitable paradigm, make a judgment based on a similar form.
  Use the following to guide paradigm selection, variable/constraint patterns, and tight bounds:
  {Domain_Knowledge}
  
   
  Your Mapper must recognize these patterns and set **fit_score** accordingly.
  
  ## Requirements
  * Mapper must strictly adhere to Extractor semantics without domain-knowledge-based simplifications. Domain Knowledge is for paradigm selection only, not for modifying specific rules.
  
  * Use Extraction content when possible (IDs, `rid`s, units, time granularity such as `{time_model.granularity}`). Map every Extraction rule to one or more **constraint templates** with `source_rule_ids` (the Extraction `rid`s).
  
  * Implicit Constraints Reasoning: Analyze the problem structure and "implicitConstraints" of domain knowledge to intelligently infer which implicit constraints should be included based on the problem characteristics and add them to the operation rules. Apply constraint templates conditionally, considering the operational context and logical requirements of the problem.
  
  * When selecting paradigms, integrate insights from **Domain_Knowledge** to weight or elucidate the fit_scores. explicitly consider computational factors such as the number of entities, time granularity, and horizon length from the Extraction. Provide clear justification for paradigm fit_scores based on task properties and resource constraints.
  
  "Constraint Fidelity Requirement": "For  operational constraints where Domain Knowledge provides detailed formulations, the Mapper must use the full formulations from without simplification. Maintain the constraints such as  boundary conditions, prohibited arcs/self loops, and special node handling exactly as specified."
  
  * cir_implementation_clusters: Aggregate CIR implementations that exhibit similarity or relatedness into clusters; each cluster shall encompass:
  
    ** class_name, rule_ids (comprising a list of rule identifiers from the Extraction phase), rationale, and relationship_strength (describing the overall strong interrelations within the cluster).
    ** top_paradigms: Propose at most three optimal paradigms (drawn from the paradigm library) ranked by suitability, with each including name, fit_score within the interval [0,1], why, strengths, and risks. 
  
  * sets_and_indices: Formally define sets (for example, use notations such as J/K/T/...when applicable) in alignment with the Extraction phase and the constraint paradigms; the set T must be annotated with {time_model.granularity}. 
  
  * parameters: Centrally declare parameters including p[j,m], r[j,m,k], R[k], B, H, M_big, etc., each accompanied by unit and source.
  
  * variable_plan: The mapper should define variables according to the characteristics of tasks from the Extraction. Organize variable families by paradigmw (e.g., time-indexed, continuous-time, event-based, arc-flow), with each variable specifying domain and meaning.
  
  * Objective fidelity & direct answer emission. The optimization objective must exactly encode the objective as stated in the Extraction (no surrogate or proxy objectives). 
  
  * Keep **units consistent**; echo time granularity from the Extraction and define `H` (horizon) if needed.
  
  * If contention_pairs_vars != empty, you MUST instantiate priority_var_pi_{p,q} for all (p,q) in contention_pairs_vars and include BOTH disjunction constraints under "Priority-driven disjunctions for contention pairs" in the Mapper output.
  
  ## Output Format

  Return **exactly one JSON object** (pretty-printed, no comments) following this schema:
  
\end{Verbatim}
\end{quote}

Output Schema: The Mapper Agent produces a JSON object adhering to the following structure:

\begin{quoting}
  \begin{Verbatim}[fontsize=\small,breaklines=true,breaksymbol={},samepage=false]
    
      "metadata": {
        "domain_tag": "{domain_tag}",
        "time_model": {"granularity": "{time_model.granularity}", "horizon_hint": "{H_if_known_or_placeholder}"},
        "notes": ["All IDs and units reflect the Extraction.",  "Objective must exactly encode the question in the Extraction; final optimization result must be the direct answer (no intermediate text)."]
      },
    
      "sets_and_indices": {
        "J": {"desc": "activities/tasks", "from": "{Extraction.sets.J}"},
        "L": {"desc": "discrete resource/level indices for value-selection (task- or resource-specific)", "from": "{Extraction.sets.L_or_levels}", "optional": true},
        "K": {"desc": "resources (renewable/non-renewable)", "from": "{Extraction.sets.K}"},
        "M": {"desc": "modes per activity", "from": "{Extraction.sets.M}"},
        "S": {"desc": "precedence pairs or scenarios", "from": "{Extraction.sets.S}"},
        "T": {"desc": "time slots if discrete", "from": "{Extraction.sets.T}", "granularity": "{time_model.granularity}"}
      },
    
      "parameters": {
        "p[j,m]": {"unit": "{time_unit}", "desc": "duration of activity j in mode m", "source": "{Extraction.params.p_jm}"},
        "r[j,m,k]": {"unit": "{resource_unit}", "desc": "demand of activity j,m on resource k", "source": "{Extraction.params.r_jmk}"},
        "R[k]": {"unit": "{resource_unit}", "desc": "capacity of resource k", "source": "{Extraction.params.R_k}"},
        "c[j,m]": {"unit": "{currency}", "desc": "mode cost (optional)", "source": "{Extraction.params.c_jm}"},
        "B": {"unit": "{currency_or_units}", "desc": "budget or non-renewable limit", "source": "{Extraction.params.B}"},
        "H": {"unit": "{time_unit}", "desc": "planning horizon upper bound", "source": "{Extraction.params.H}"},
        "M_big": {"unit": "{time_unit}", "desc": "tight big-M; use ES/LS bounds", "source": "{Extraction.params.M_big}"}
      },
    
      "variable_plan": {
        "time_indexed": {
          "x[j,m,t]": {"domain": "{0,1}", "meaning": "activity j in mode m starts at time t"},
          "z[j,m,t]": {"domain": "{0,1}", "meaning": "activity j,m active at time t (convolution of x and p)"},
          "C_max": {"domain": "R_+", "meaning": "makespan"}
        },
        "continuous_time": {
          "S[j]": {"domain": "R_+", "meaning": "start time"},
          "C[j]": {"domain": "R_+", "meaning": "completion time"},
          "y[i,j,k]": {"domain": "{0,1}", "meaning": "on resource k, i precedes j"}
        },
        "event_based": {
          "E[e]": {"domain": "R_+", "meaning": "time of event e"},
          "a[j,e]": {"domain": "{0,1}", "meaning": "activity j assigned to event e"}
        },
        "arc_flow": {
          "x[i,j]": {"domain": "{0,1}", "meaning": "select compatibility arc i->j in DAG"},
          "path_used": {"domain": "Z_+", "meaning": "number of selected s->t paths if needed"}
        }
      },
      "constraint_templates": [
        {
          "rule_ref": "precedence",
          "source_rule_ids": ["{rid_precedence}"],
          "intent": "a task cannot start before all predecessors complete",
          "time_indexed_form": "\\forall(i,j) in Pred: \\sum_{m,t} (t + p[i,m]) x[i,m,t] \\leq \\sum_{m,t} t x[j,m,t]",
          "continuous_time_form": "\\forall(i,j) in Pred: S[j] \\geq C[i]",
          "event_based_form": "Event index of j not earlier than event index end of i"
        },
        {
          "rule_ref": "reliability_at_least_M",
          "source_rule_ids": ["{rid_reliability}"],
          "intent": "achieve at least M successful outcomes with maximum probability or guaranteed confidence",
          "time_indexed_form": "Define per-target success probability P_j via value-selection over (bomb count, type); compute Poisson-binomial tail Pr(\\Sigma Z_j \\geq M) exactly via convolution or approximate via piecewise linear bounds.",
          "continuous_time_form": "Same P_j construction; reliability either in the objective (maximize) or via chance constraint Pr(\\Sigma Z_j \\geq M) \\geq 1 - \\alpha using Boole/Bonferroni or CVaR-type linearizations."
        },
        {
          "rule_ref": "renewable_capacity",
          "source_rule_ids": ["{rid_capacity}"],
          "intent": "at any time, total usage of each renewable resource does not exceed capacity",
          "time_indexed_form": "\\forall t,k: \\sum_{j,m} r[j,m,k]\\cdot z[j,m,t] \\leq R[k]",
          "continuous_time_form": "Profile inequality: \\sum_{j active at t} r[j,m_j,k] \\leq R[k]"
        },
        {
          "rule_ref": "discrete_multiple_of_q",
          "source_rule_ids": ["{rid_capacity_or_integrality_from_Extraction}"],
          "intent": "encode x in Z*q with gating and capacity on the true scaled quantity",
          "continuous_time_form": "Introduce u in Z_+, and enforce x = q*u. If on/off y gates x, bound u by 0 <= u <= U*y and (optionally) u >= y when positive minimum is needed; impose capacity as sum(q*u) <= Cap. **Do NOT write sum(x*y) <= Cap.**"
        },
        {
          "rule_ref": "cartesian_value_selection_bw_compute",
          "source_rule_ids": ["{rid_capacity}", "{rid_modes}"],
          "intent": "select a joint (bandwidth, compute) level for offloaded tasks; precompute service times; enforce caps without x*y products",
          "continuous_time_form": "Define B_vals and K_vals; add w[j,b,k] in {0,1}. sum_{b,k} w[j,b,k] = y_offload[j]. Service time: t_scn[j] = t_loc[j]*(1-y_offload[j]) + sum_{b,k} t_off[j,scn,b,k]*w[j,b,k]. Caps: sum_{j,b,k} b*w[j,b,k] <= B_max; sum_{j,b,k} (c_chunk*k)*w[j,b,k] <= C_max.",
          "time_indexed_form": "same selection with period-usage if needed",
          "notes": "avoid variables divided by decisions and avoid capacity terms like \\sum(BW[j]\\cdot offload[j]) or \\sum(C[j]\\cdot offload[j])"
        },
        {
          "rule_ref": "non_preemption_and_time_windows",
          "source_rule_ids": ["{rid_nonpreempt}", "{rid_windows}"],
          "intent": "each activity runs non-preemptively within its allowed window",
          "time_indexed_form": "x[j,m,t] = 0 outside [ES[j], LS[j]]; z derived by convolution; no splitting",
          "continuous_time_form": "S[j] in [ES[j], LS[j]], C[j] = S[j] + p[j,m_j]"
        },
        {
          "rule_ref": "mode_selection",
          "source_rule_ids": ["{rid_modes}"],
          "intent": "choose exactly one mode per activity",
          "time_indexed_form": "\\forall j: \\sum_{m,t} x[j,m,t] = 1",
          "continuous_time_form": "Binary mode selectors with SOS1 or assignment"
        },
        {
          "rule_ref": "non_renewable_budget",
          "source_rule_ids": ["{rid_budget}"],
          "intent": "consumption of non-renewables within budget",
          "time_indexed_form": "\\sum_{j,m} r[j,m,k]\\cdot(\\sum_t x[j,m,t]) \\leq R[k] for non-renewable k"
        },
        {
          "rule_ref": "makespan_definition",
          "source_rule_ids": ["{rid_makespan}"],
          "intent": "define C_max",
          "time_indexed_form": "\\forall j: C_max \\geq \\sum_{m,t} (t + p[j,m]) x[j,m,t]",
          "continuous_time_form": "\\forall j: C_max \\geq C[j]"
        }
      ],
      "constraint_clusters": [
        {
          "class_name": "precedence_and_execution_timing",
          "rule_ids": ["{rid_precedence}", "{rid_nonpreempt}", "{rid_windows}"],
          "relationship_strength": "strong",
          "rationale": "jointly determine feasible start/finish structure",
          "top_paradigms": [
            {
              "name": "continuous_time_milp",
              "fit_score": 0.0,
              "why": "compact time, strong for precedence + non-overlap with tight windows",
              "strengths": ["compact time scale","valid inequalities available"],
              "risks": ["needs tight big-M/indicators"]
            },
            {
              "name": "time_indexed_milp",
              "fit_score": 0.0,
              "why": "no big-M for overlap; transparent windows via ES/LS",
              "strengths": ["no-overlap without big-M","easy window filtering"],
              "risks": ["scales with |T|"]
            },
            {
              "name": "event_based_milp",
              "fit_score": 0.0,
              "why": "few distinct events; precedence-dominant",
              "strengths": ["fewer indices","sequence-centric"],
              "risks": ["event count selection delicate"]
            }
          ]
        },
        {
          "class_name": "renewable_and_nonrenewable_capacity",
          "rule_ids": ["{rid_capacity}", "{rid_budget}"],
          "relationship_strength": "strong",
          "rationale": "dynamic renewable limits and cumulative budgets",
          "top_paradigms": [
            {
              "name": "time_indexed_milp",
              "fit_score": 0.0,
              "why": "slot-wise capacity accounting is direct and strong",
              "strengths": ["exact capacity profile"],
              "risks": ["large |T|"]
            },
            {
              "name": "cumulative_constraint_cp",
              "fit_score": 0.0,
              "why": "global cumulative handles calendars and side constraints",
              "strengths": ["powerful propagation"],
              "risks": ["weak LP bounds; performance variance"]
            },
            {
              "name": "bucket_indexed_energy_model",
              "fit_score": 0.0,
              "why": "long horizons where aggregation acceptable",
              "strengths": ["smaller model size"],
              "risks": ["intra-bucket approximation"]
            }
          ]
        },
        {
          "class_name": "mode_and_cost_structure",
          "rule_ids": ["{rid_modes}", "{rid_costs}"],
          "relationship_strength": "medium",
          "rationale": "mode choice impacts duration/capacity and costs",
          "top_paradigms": [
            {
              "name": "time_indexed_milp",
              "fit_score": 0.0,
              "why": "mode–time coupling is explicit via x[j,m,t]",
              "strengths": ["clear mode selection"],
              "risks": ["variable count growth"]
            },
            {
              "name": "continuous_time_milp",
              "fit_score": 0.0,
              "why": "mode via binary selectors with linear timings",
              "strengths": ["compact time representation"],
              "risks": ["mode-dependent big-M tuning"]
            },
            {
              "name": "arc_flow_path_decomposition",
              "fit_score": 0.0,
              "why": "when modes imply alternative feasible arcs/paths",
              "strengths": ["implicit non-overlap via arcs","CG-friendly"],
              "risks": ["compatibility graph size"]
            }
          ]
        }
      ],
    "objective_template": {
    "primary": "{exact_objective_from_Extraction_no_surrogates}",
    "unit_conversion": "All units must be consistent with Extraction. If unit conversion is required, use: conversion_factor = target_unit/source_unit",
    "direct_answer": "The optimization result must be the final answer without additional conversion steps",
        "examples": [
          {"name": "maximize_reliability_M_of_N", "expression": "maximize P{\\Sigma_j Z_j \\geq M} or enforce P{\\Sigma_j Z_j \\geq M} \\geq 1 - \\alpha"},
          {"name": "minimize_makespan", "expression": "min C_max"},
          {"name": "minimize_total_cost", "expression": "mode costs + resource/penalty costs"},
          {"name": "minimize_expected_makespan", "expression": "\\sum_s \\pi_s \\cdot C_max^s"}
        ]
      },
    
      "graph_schema": {
        "applicability": "high for compatibility networks or depot-style flows",
        "node_types": [{"name": "source", "singleton": true}, {"name": "task"}, {"name": "sink", "singleton": true}],
        "edge_types": [
          {"name": "compatibility", "from": "task", "to": "task", "feasibility_condition": "finish(i)+setup(i->j) \\leq start(j)"}
        ],
        "edge_cost_recipe": {"compatibility": "setup/idle/transition costs as specified in Extraction"}
      },
    
      "time_indexing": {
        "granularity": "{time_model.granularity}",
        "T_definition": "{if_discrete: e.g., T = 0..H}",
        "feasible_start_filter": "apply ES/LS windows from Extraction to restrict x[j,m,t]"
      }


  \end{Verbatim}
\end{quoting}

\subsection{The Formalizer Agent}
\label{app:formalizer}

The Formalizer Agent converts abstract constraint archetypes and paradigm recommendations into 
executable mathematical formulations. 
Below we present the complete system prompt used for the Formalizer Agent:

\begin{quote}
\begin{Verbatim}[fontsize=\small,breaklines=true,breaksymbol={},samepage=false]
  You are an expert operations research modeler.

  I will provide you the Extraction and Mapper of the problem.

  Original problem description:
  {Original_problem_description}
  
  Extraction: a strict JSON produced by the Extractor (contains domain tags, problem summary, entities, time model, measures, explicit_rules with rid, given_parameters, objective, tabular_values, etc.).
  {Extraction}
  
  Mapper: a strict JSON produced by the Mapper (contains metadata, variable plans, constraint templates with source_rule_ids that reference Extraction rids, sets/indices, parameters, operation_rule_clusters with prefer paradigm, relationship_matrix, objective_template, graph/time indexing recipes).
  {Mapper}
  
  
  
  Your tasks:
  1. Present the complete mathematical formulation.
  2. Provide executable Gurobi Python code consistent with the mathematical formulation and with the parameters/data implied by the problem description and Extraction.
  
  
  ## Output Format:
  Use exactly these markers and structure:
  
  #problem_formulation_start#
  [Complete mathematical formulation with sets, variables, parameters, constraints, objective]
  #problem_formulation_end#
  
  #Gurobi_code_start#
  [Complete Gurobi Python code implementing the mathematical formulation]
  #Gurobi_code_end#
  
\end{Verbatim}
\end{quote}

\subsection{Checker Agent Verification Prompts}
\label{app:checker-prompts}

The Checker Agent performs validation through three distinct phases, each with targeted verification prompts. These prompts are executed after the respective agent completes its main task to catch inconsistencies and prevent error propagation.

\paragraph{Extractor-to-Mapper Verification}

\begin{quote}
\begin{Verbatim}[fontsize=\small,breaklines=true,breaksymbol={},samepage=false]
EXTRACTOR VALIDATION: Check for internal consistency in extracted data:

1. Verify all extracted parameters (time windows, resource demands, etc.) 
   align with problem description (e.g., non-negative values, valid ranges)
2. Ensure entity counts match problem description scope

If any inconsistency is found, output "VALIDATION FAILED: [specific issue]"
Otherwise, output "VALIDATION PASSED"
\end{Verbatim}
\end{quote}

\paragraph{Mapper-to-Formalizer Verification}

\begin{quote}
\begin{Verbatim}[fontsize=\small,breaklines=true,breaksymbol={},samepage=false]
MAPPER VALIDATION: Audit CIR implementation completeness:

For each operation rule (by rule ID) from the Extractor output:
1. Check that at least one CIR implementation exists
2. Verify each CIR has a specified modeling paradigm

If any rule lacks CIR implementation, output "VALIDATION FAILED: Rule [rule ID] missing CIR implementation"
Otherwise, output "VALIDATION PASSED"
\end{Verbatim}
\end{quote}

\paragraph{Formalization Verification}

\begin{quote}
\begin{Verbatim}[fontsize=\small,breaklines=true,breaksymbol={},samepage=false]
FORMALIZER VALIDATION: Verify constraint translation completeness:

Cross-reference each operation rule (explicit and validated implicit):
1. Ensure each rule has at least one corresponding mathematical constraint
2. Verify all constraints are implemented in the generated code

If any rule lacks constraint implementation, output "VALIDATION FAILED: Rule [rule ID] missing constraint"
Otherwise, output "VALIDATION PASSED"
\end{Verbatim}
\end{quote}

\subsection{Standard Prompting Baseline}
\label{app:standard-prompt}

For comparison purposes, we also evaluate proprietary LLMs (GPT-5 and Grok-4) using standard prompting. Below we present the complete system prompt used for the standard prompting baseline:

\begin{quote}
\begin{Verbatim}[fontsize=\small,breaklines=true,breaksymbol={},samepage=false]
You are an expert operations research modeler. 

I will provide you the problem description.

Problem description: {problem_statement}

you should help me to do the following tasks:
First, present the complete mathematical formulation, 
and then, provide executable Gurobi Python code with the parameter in the problem description.

## Output Format:

Use exactly these markers and structure:

#problem_formulation_start#
[Problem formulation]
#problem_formulation_end#

#Gurobi_code_start#
[Complete Gurobi code with consistent content in problem and formulation]
#Gurobi_code_end#
\end{Verbatim}
\end{quote}

\section{CIR domain knowledge construction}
\label{app:cir}
Following prior studies, our CIR construction spans multiple key optimization 
domains. \cref{tab:cir-domains} summarizes the abbreviations and references for these 
domain classes. For each domain, we analyzed canonical problem structures, 
common operational rules, and typical constraint patterns based on high-quality literature review papers. 

\begin{table}[ht]
\caption{Abbreviations, Full Names, and References for the CIR Domain Classes}
\label{tab:cir-domains}
\centering
\begin{small}
\begin{tabular}{lll}
\toprule
\textbf{Abbreviation} & \textbf{Full Name} & \textbf{Reference} \\
\midrule
CA & Crew Assignment & \cite{heil2020railway} \\
EDUS & Education & \cite{pillay2016review} \\
ES & Energy System & \cite{merkert2015scheduling} \\
HA & Healthcare Allocation & \cite{cayirli2003outpatient} \\
SCP & Supply Chain and Production & \cite{jans2008modeling} \\
PM & Project Management & \cite{sanchez2023resource} \\
RA & Resource Allocation & \cite{zhan2015cloud} \\
SPOS & Sports Planning & \cite{ribeiro2012sports} \\
JS & Job Shop & \cite{dauzere2024flexible} \\
TS & Transportation & \cite{perumal2022electric} \\
\bottomrule
\end{tabular}
\end{small}
\end{table}

\subsection{CIR Examples}
\label{app:cir:examples}

This subsection presents examples of CIR domain knowledge construction for different optimization problem domains.

\subsubsection{Transportation domain CIR Example}
\label{app:cir:transportation}

\begin{tcolorbox}[
  colback=gray!5!white,
  colframe=gray!75!black,
  boxrule=0.5pt,
  arc=2pt,
  left=6pt,
  right=6pt,
  top=6pt,
  bottom=6pt,
  breakable
]
\textbf{Intent:} Temporal consistency with depot boundaries

\textbf{Continuous-Time Formulation:}
\begin{itemize}
  \item $\forall k \in \mathcal{K}, \forall j \in \mathcal{C}: \text{node\_visit\_time}[j,k] \geq \text{depot\_start\_time}[k] + \text{deadhead\_time}[\text{depot},j] - M_{\text{big}} \cdot (1 - \text{vehicle\_flow}[\text{depot},j,k])$
  
  \item $\forall k \in \mathcal{K}, \forall i,j \in \mathcal{C}, i \neq j: \text{node\_visit\_time}[j,k] \geq \text{node\_visit\_time}[i,k] + \text{service\_duration}[i] + \text{deadhead\_time}[i,j] - M_{\text{big}} \cdot (1 - \text{vehicle\_flow}[i,j,k])$
  
  \item $\forall k \in \mathcal{K}, \forall i \in \mathcal{C}: \text{depot\_end\_time}[k] \geq \text{node\_visit\_time}[i,k] + \text{service\_duration}[i] + \text{deadhead\_time}[i,\text{depot}] - M_{\text{big}} \cdot (1 - \text{vehicle\_flow}[i,\text{depot},k])$
  
  \item $\forall k \in \mathcal{K}: \text{depot\_end\_time}[k] \geq \text{depot\_start\_time}[k]$
  
  \item $\forall k \in \mathcal{K}: \text{depot\_start\_time}[k] \geq \text{time\_window\_lower}[\text{depot}]$
  
  \item $\forall k \in \mathcal{K}: \text{depot\_end\_time}[k] \leq \text{time\_window\_upper}[\text{depot}]$
\end{itemize}

\textbf{Time-Indexed Formulation:} (Not applicable for this constraint)

\begin{quote}
\textit{Notes:} For single-depot routing problems, use two separate time variables to represent the departure time and return time of vehicles from the depot to avoid time conflicts. $\text{depot\_start\_time}[k]$ denotes the time when vehicle $k$ departs from the depot, and $\text{depot\_end\_time}[k]$ denotes the time when vehicle $k$ returns to the depot.
\end{quote}
\end{tcolorbox}

\subsubsection{Job Shop domain CIR Example}
\label{app:cir:jobshop}

\begin{tcolorbox}[
  colback=gray!5!white,
  colframe=gray!75!black,
  boxrule=0.5pt,
  arc=2pt,
  left=6pt,
  right=6pt,
  top=6pt,
  bottom=6pt,
  breakable
]
\textbf{Intent:} Machine capacity (no-overlap) constraints

\textbf{Continuous-Time Formulation:}
\begin{itemize}
  \item $\forall i \in \mathrm{StageSet}, j \neq j': \mathrm{completionTime}_{j,i} \geq \mathrm{completionTime}_{j',i} + \mathrm{procTime}_{j,i} - \mathrm{bigM} \cdot (1 - \mathrm{orderAfterPair}_{i,j,j'})$
  
  \item $\forall i \in \mathrm{StageSet}, j \neq j': \mathrm{completionTime}_{j',i} \geq \mathrm{completionTime}_{j,i} + \mathrm{procTime}_{j',i} - \mathrm{bigM} \cdot \mathrm{orderAfterPair}_{i,j,j'}$
  
  \item $\forall i \in \mathrm{StageSet}, j \neq j', k \in \mathrm{MachineSet}_i: \mathrm{completionTime}_{j,i} \geq \mathrm{completionTime}_{j',i} + \mathrm{procTime}_{j,i} - \mathrm{bigM} \cdot (3 - \mathrm{orderAfterPair}_{i,j,j'} - \mathrm{assignJobToMachine}_{j,i,k} - \mathrm{assignJobToMachine}_{j',i,k})$
\end{itemize}

\textbf{Time-Indexed Formulation:}
\begin{itemize}
  \item $\forall i \in \mathrm{StageSet}, t \in \mathrm{TimeSlots}: \sum_{j \in \mathrm{JobSet}} \mathrm{processAtTime}_{j,i,t} \leq |\mathrm{MachineSet}_i|$
  
  \item $\forall i \in \mathrm{StageSet}, k \in \mathrm{MachineSet}_i, t \in \mathrm{TimeSlots}: \sum_{j \in \mathrm{JobSet}} \mathrm{processOnMachineAtTime}_{j,i,k,t} \leq 1$
\end{itemize}

\begin{quote}
\textit{Notes:} Disjunctive (big-M) or cumulative time-indexed constraints ensure capacity is not exceeded. Pairwise order variables apply only when operations share a machine.
\end{quote}
\end{tcolorbox}

\subsubsection{Healthcare domain CIR Example}
\label{app:cir:healthcare}

\begin{tcolorbox}[
  colback=gray!5!white,
  colframe=gray!75!black,
  boxrule=0.5pt,
  arc=2pt,
  left=6pt,
  right=6pt,
  top=6pt,
  bottom=6pt,
  breakable
]
\textbf{Intent:} Shift Assignment Limits

\textbf{Continuous-Time Formulation:}
\begin{itemize}
  \item $\sum_{s\in\mathcal{Shifts}} \mathrm{assignShift}_{i,d,s} \leq 1 \quad \forall i\in\mathcal{Staff}, d\in\mathcal{Days}$
\end{itemize}

\textbf{Time-Indexed Formulation:}
\begin{itemize}
  \item $\sum_{s\in\mathcal{Shifts}} \mathrm{assignShift}_{i,d,s} \leq 1 \quad \forall i\in\mathcal{Staff}, d\in\mathcal{Days}$
\end{itemize}

\begin{quote}
\textit{Notes:} Each staff member is assigned at most one shift per day, ensuring basic workforce scheduling feasibility. This constraint is fundamental to personnel rostering across healthcare domains.
\end{quote}
\end{tcolorbox}

\section{ORCOpt-Bench: Operation Rule to Constraint Optimization Benchmark}
\label{app:orcoopt}

To systematically evaluate the capability of LLMs in translating natural language descriptions 
of optimization problems into mathematical models and executable code, we introduce the Operation 
Rule to Constraint Optimization Benchmark (ORCOpt-Bench). Unlike domain-specific benchmarks, ORCOpt-Bench 
spans multiple optimization domains, testing the generalization capability of LLMs across different problem types. The ORCOpt-Bench was constructed through a two-step process:

\textbf{Seed Problem Collection:} We created 20 optimization problems across the 10 domains 
listed in \cref{tab:orcopt-bench}, drawing from both academic literature and real-world 
industry cases obtained through our previous projects.
The benchmark comprises problems selected for their operational rules requiring decomposition 
into multiple constraints or exhibiting dependencies on other rules or parameters.
By integrating problems from academic literature and industrial case studies, 
the benchmark balances theoretical rigor with practical applicability.
To prevent data leakage and assess generalization, the benchmark 
problems (\cref{tab:orcopt-bench}) are sourced from application papers 
distinct from the review articles used for CIR construction (\cref{tab:cir-domains}).

\begin{table}[ht]
\caption{Problem Classes and References in ORCOpt-Bench}
\label{tab:orcopt-bench}
\centering
\begin{small}
\begin{tabular}{ll}
\toprule
\textbf{Problem} & \textbf{Reference} \\
\midrule
Airline Crew Rostering & \cite{maenhout2010hybrid} \\
University Exam Timetabling & \cite{bazari2023modeling} \\
Energy Demand Management & \cite{tuysuz2024real} \\
Diagnostic Facility Management & \cite{green2006managing} \\
Multi-Level Lot Sizing & \cite{sahling2009solving} \\
Resource-Constrained Project Management & \cite{deblaere2011proactive} \\
UAV Edge Computing Networks & \cite{zhou2023priority} \\
Sports Event Operations & \cite{zhang2025optimization} \\
Autonomous Manufacturing Systems & \cite{iwamura2010study} \\
Vehicle routing/Activity Chain Optimization & \cite{esztergar2021toward} \\
\bottomrule
\end{tabular}
\end{small}
\end{table}

\textbf{Augmentation and Variant Generation:} To enhance the diversity, robustness, and difficulty 
spectrum of the benchmark, we applied systematic augmentation strategies to the seed problems:

•	Constraint Addition/Removal: We modified the complexity of problems by adding or removing 
one or two core operational rules. This creates a gradient of difficulty, testing the framework's 
ability to handle both simplified scenarios (with fewer core rules) and more constrained, complex 
scenarios.

•	Objective Transformation: We altered the problem objectives to assess generalization across 
different performance measures. This includes transforming the goal, for example, from pure cost 
minimization to profit maximization, or by introducing penalty terms for soft constraints.

This process resulted in a final benchmark comprising 50 problems, primarily composed of 
Mixed-Integer Linear Programming (MILP) and Quadratic Programming (QP) formulations. 
The dataset's composition is illustrated in \cref{fig:benchmark-distribution} 
and \cref{fig:constraint-types-distribution}, which 
visualizes the distribution of problems across domains and constraint 
category counts, highlighting the structural diversity and complexity of the benchmark.

\begin{figure}[ht]
  \setlength{\abovecaptionskip}{0pt}
  \setlength{\belowcaptionskip}{5pt}
  \vspace{-0.1in}
  \begin{center}
    \includegraphics[width=0.5\columnwidth]{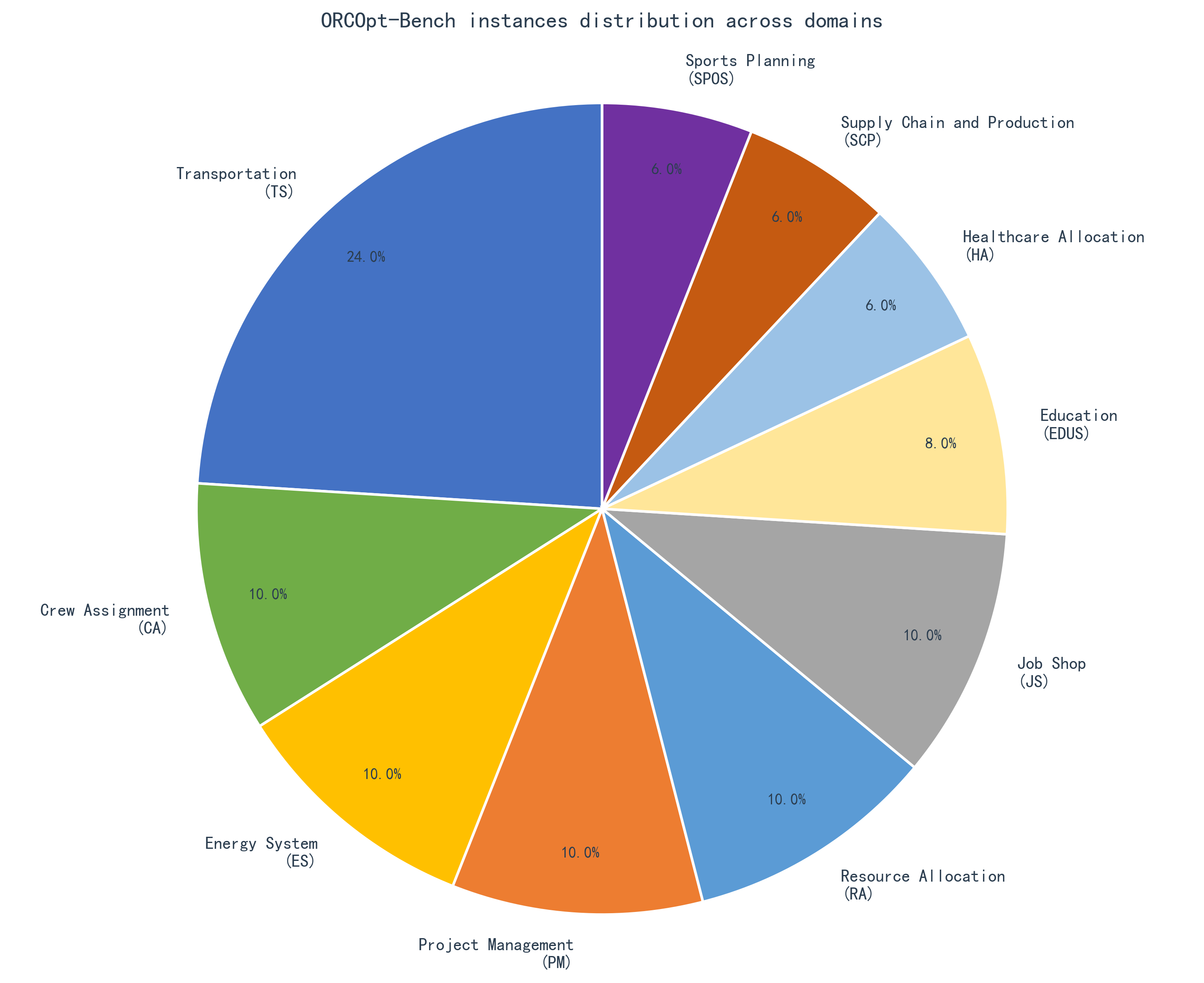}
    \caption{
      Distribution of problems across domains and constraint category counts in ORCOpt-Bench.
    }
    \label{fig:benchmark-distribution}
  \end{center}
  \vspace{-0.1in}
\end{figure}

\begin{figure}[ht]
  \centering
  \includegraphics[width=0.8\textwidth]{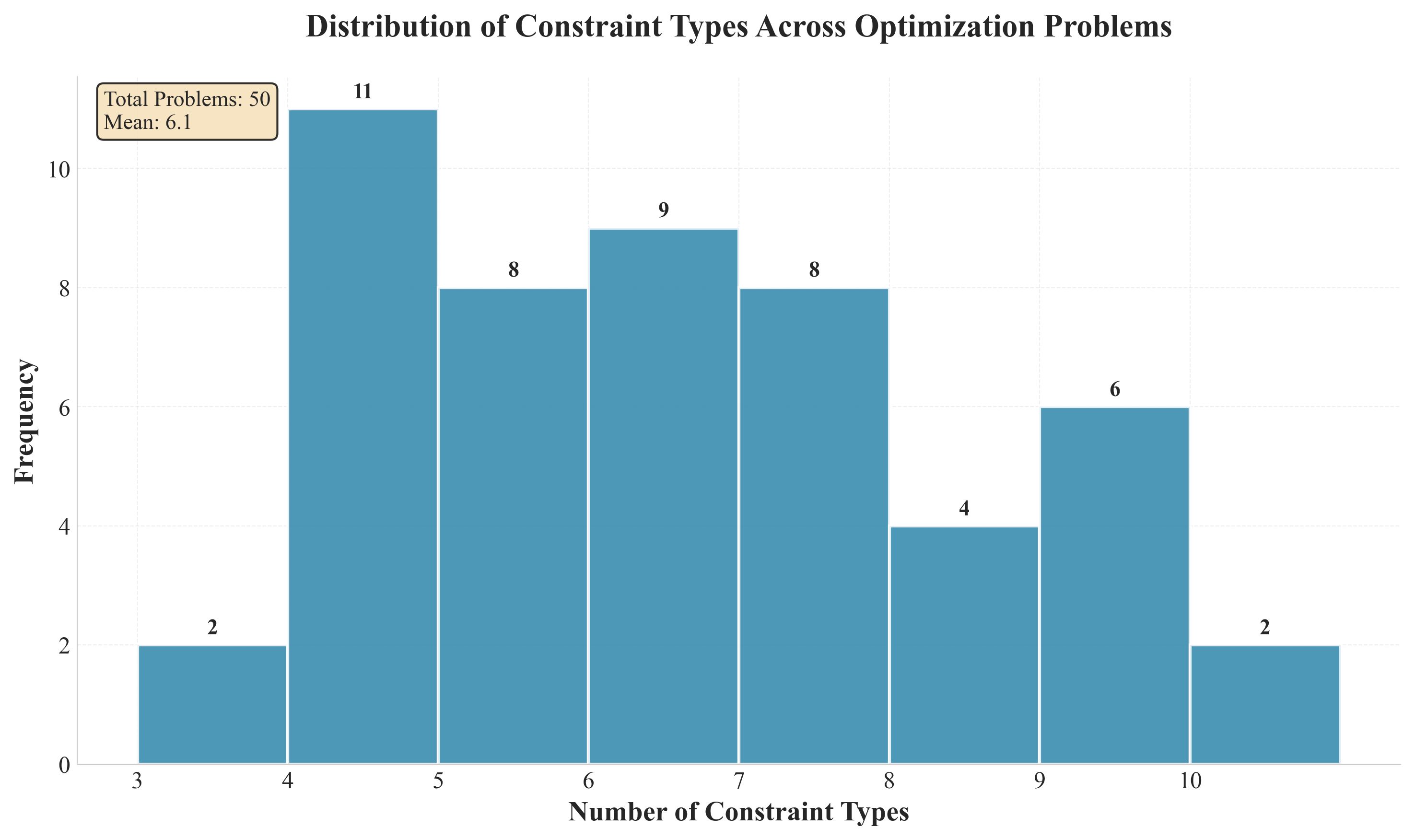}
  \caption{Distribution of constraint types per problem in ORCOpt-Bench. The figure shows the statistical distribution of the number of constraint types contained in each problem in the dataset.}
  \label{fig:constraint-types-distribution}
\end{figure}

The benchmark is provided in the supplementary material. 
Examples of operation rule generation from problem descriptions are presented in \autoref{app:instance-diversity}.

\section{Instance code diversity}
\label{app:instance-diversity}

In Section 4.3, we noted that certain problem instances, while successfully solved by the R2C framework, 
exhibited modeling approaches that diverged from our reference formulations. These cases are not failures 
of the framework; rather, they highlight its capacity to discover valid and often innovative modeling pathways 
that still satisfy all operational rules and achieve optimal or near-optimal solutions. In this appendix, we 
present three such illustrative cases, each drawn from a distinct domain within ORCOpt-Bench. For each case, 
we compare the modeling strategy generated by R2C against the reference approach, analyze the semantic equivalence 
of the constraints, and discuss the implications for both model performance and interpretability. This diversity 
in solution strategies underscores that R2C does not rely on rote memorization of templates, but instead performs 
genuine reasoning over rule semantics and paradigm selection.

\textbf{Case 1:} Problem 41-Heterogeneous Vehicle Routing Problem with Time Windows (HVRPTW) featuring cross-district penalties. 

The Extractor agent successfully identified the operation rule regarding capacity limits. 
The Mapper agent subsequently synthesized this rule into different CIRs. Although both CIR 
implementations adhere to the semantic intent of the original operation rule—"Vehicle loading 
cannot exceed volume capacity limits"—they instantiate this requirement through fundamentally 
different mathematical structures.

\subsubsection*{Formulation 1 (Direct Aggregation)}

\textbf{CIR implementation:}
\begin{quoting}
\begin{Verbatim}[fontsize=\small,breaklines=true,breaksymbol={},samepage=false]
{
  "constraint_templates": [
    {
      "rule_ref": "capacity_constraint",
      "source_rule_ids": ["R3"],
      "intent": "vehicle loading cannot exceed capacity",
      "continuous_time_form": "forall k in K: sum_{j in J\\{Depot}} demand[j]*(sum_{i in J} x[i,j,k]) <= vehicle_capacity[k]",
      "computational_properties": {
        "constraint_count": "O(|K|)",
        "complexity_class": "linear",
        "lp_relaxation": "weaker bounds",
        "extensibility": "limited to simple capacity constraints"
      }
    }
  ],
  "variable_plan": {
    "continuous_time": {
      "x[i,j,k]": {
        "domain": "{0,1}",
        "meaning": "vehicle k travels directly from i to j"
      }
    }
  }
}
\end{Verbatim}
\end{quoting}
\textbf{Mathematical Formulation:}
\begin{equation}
\forall k \in K: \sum_{j \in J \setminus \{\text{Depot}\}} d_j \cdot \left(\sum_{i \in J} x_{ijk}\right) \leq Q_k
\end{equation}

\textbf{Implementation Code:}
\begin{lstlisting}[language=Python]
# CIR Template: forall k in K: sum_{j in J\{Depot\}} demand[j]*(sum_{i in J} x[i,j,k]) <= vehicle_capacity[k]
m.addConstr(quicksum(demand[i]*y[i,k] for i in N) <= veh[k]['cap'], name=f'cap_{k}')
\end{lstlisting}

\subsubsection*{Formulation 2 (Load Propagation)}

\textbf{CIR implementation:}
\begin{quoting}
\begin{Verbatim}[fontsize=\small,breaklines=true,breaksymbol={},samepage=false]
{
  "constraint_templates": [
    {
      "rule_ref": "load_initialization",
      "source_rule_ids": ["R3"],
      "intent": "initialize load at depot",
      "continuous_time_form": "forall k in K: load_level[Depot,k] = 0",
      "computational_properties": {
        "constraint_count": "O(|K|)", 
        "complexity_class": "linear",
        "role": "foundation for load tracking"
      }
    },
    {
      "rule_ref": "load_propagation", 
      "source_rule_ids": ["R3"],
      "intent": "propagate load along vehicle routes",
      "continuous_time_form": "forall i,j in J,k in K: load_level[j,k] >= load_level[i,k] + demand[j] - M_big*(1 - x[i,j,k])",
      "computational_properties": {
        "constraint_count": "O(|V|²|K|)",
        "complexity_class": "quadratic", 
        "lp_relaxation": "tighter bounds",
        "extensibility": "high - accommodates complex loading sequences"
      }
    },
    {
      "rule_ref": "capacity_enforcement",
      "source_rule_ids": ["R3"], 
      "intent": "enforce vehicle capacity limits",
      "continuous_time_form": "forall j in J,k in K: load_level[j,k] <= vehicle_capacity[k]",
      "computational_properties": {
        "constraint_count": "O(|V||K|)",
        "complexity_class": "linear",
        "role": "distributed capacity enforcement"
      }
    }
  ],
  "variable_plan": {
    "continuous_time": {
      "x[i,j,k]": {
        "domain": "{0,1}", 
        "meaning": "vehicle k travels directly from i to j"
      },
      "load_level[j,k]": {
        "domain": "R_+",
        "meaning": "cumulative load of vehicle k at node j"
      }
    }
  }
}
\end{Verbatim}
\end{quoting}
\textbf{Mathematical Formulation:}
\begin{align}
\forall i,j \in V, k \in K: &\quad u_{jk} \geq u_{ik} + d_j - M(1 - x_{ijk}) \\
\forall i \in V, k \in K: &\quad u_{ik} \leq Q_k \\
&\quad u_{\text{depot},k} = 0
\end{align}

\textbf{Implementation Code:}
\begin{lstlisting}[language=Python]
# Extended CIR: Load tracking along routes
m.addConstr(u[j,k] >= u[i,k] + vol[j] - M*(1 - x[i,j,k]), name=f"loadprop_{i}_{j}_{k}")
\end{lstlisting}

\textbf{Computational Analysis:} From a computational perspective, the direct aggregation method generates only a linear number of constraints ($O(|K|)$) that enforce capacity compliance, resulting in a more compact model. In contrast, the load propagation approach, which tracks cumulative load along routes through auxiliary variables, creates a quadratic number of constraints ($O(|V|^2|K|)$) that increase model size. However, this formulation typically provides tighter LP relaxation bounds that may accelerate branch-and-bound solving. Furthermore, the load propagation model offers superior extensibility by naturally accommodating complex loading sequences or sequence-dependent constraints due to its granular representation of the physical loading process.

\textbf{Case 2:} Crew Rostering Problem with Skill Requirements featuring Fairness and Preference Constraints.

The Extractor agent successfully identified the operation rule regarding skill requirements. The Mapper agent subsequently synthesized this rule into different CIRs. Although both CIR implementations adhere to the semantic intent of the original operation rule—``Activity skill requirements must be met with possible understaffing''—they instantiate this requirement through fundamentally different mathematical structures.

\subsubsection*{Formulation 1 (Direct Aggregation)}

\textbf{CIR Template:}
\begin{quoting}
\begin{Verbatim}[fontsize=\small,breaklines=true,breaksymbol={},samepage=false]
{
  "constraint_templates": [
    {
      "rule_ref": "skill_requirement_constraint",
      "source_rule_ids": ["R1"],
      "intent": "activity skill requirements must be met with possible understaffing",
      "continuous_time_form": "forall j in J, s in S: sum_{k in K: skill[k]=s} x[j,k] + u[j,s] = req[j,s]",
      "computational_properties": {
        "constraint_count": "O(|J|×|S|)",
        "complexity_class": "linear",
        "lp_relaxation": "weaker bounds due to continuous understaffing variables",
        "extensibility": "limited to simple skill requirement constraints"
      }
    }
  ],
  "variable_plan": {
    "continuous_time": {
      "x[j,k]": {
        "domain": "{0,1}",
        "meaning": "crew k is assigned to activity j"
      },
      "u[j,s]": {
        "domain": "R_+",
        "meaning": "understaffed hours for activity j and skill s"
      }
    }
  }
}
\end{Verbatim}
\end{quoting}

\textbf{Mathematical Formulation:}
\begin{equation}
\forall j \in J, s \in S: \sum_{k \in K: \text{skill}[k]=s} x_{jk} + u_{js} = R_{js}
\end{equation}

\textbf{Implementation Code:}
\begin{lstlisting}[language=Python]
# CIR Template: forall j in J, s in S: sum_{k in K} [skill[k]=s]*x[j,k] + u[j,s] = req[j,s]
for j in activities:
    for s in ['Cabin','Service']:
        model.addConstr(
            gp.quicksum(x[k, j] for k in crew if skill[k] == s) + u[j, s] == req[j, s],
            f"skill_req_{j}_{s}"
        )
\end{lstlisting}

\subsubsection*{Formulation 2 (Indicator Variables with Big-M)}

\textbf{CIR Template:}
\begin{quoting}
\begin{Verbatim}[fontsize=\small,breaklines=true,breaksymbol={},samepage=false]
{
  "constraint_templates": [
    {
      "rule_ref": "skill_satisfaction_indicator",
      "source_rule_ids": ["R1"],
      "intent": "track whether cabin skill requirement is satisfied",
      "continuous_time_form": "forall j in J: met_cabin[j] <= sum_{k in K_cabin} x[j,k], met_cabin[j] >= (sum_{k in K_cabin} x[j,k])/M",
      "computational_properties": {
        "constraint_count": "O(|J|)",
        "complexity_class": "linear",
        "role": "foundation for skill requirement tracking"
      }
    },
    {
      "rule_ref": "understaffing_logic", 
      "source_rule_ids": ["R1"],
      "intent": "determine understaffing status based on skill satisfaction",
      "continuous_time_form": "forall j in J: understaffed[j] >= 1 - met_cabin[j], understaffed[j] >= 1 - met_service[j], understaffed[j] <= 2 - met_cabin[j] - met_service[j]",
      "computational_properties": {
        "constraint_count": "O(|J|)",
        "complexity_class": "linear",
        "lp_relaxation": "tighter bounds with proper M value",
        "extensibility": "moderate - accommodates complex skill logic"
      }
    }
  ],
  "variable_plan": {
    "continuous_time": {
      "x[j,k]": {
        "domain": "{0,1}", 
        "meaning": "crew k is assigned to activity j"
      },
      "met_cabin[j]": {
        "domain": "{0,1}",
        "meaning": "indicator for cabin skill requirement satisfaction"
      },
      "met_service[j]": {
        "domain": "{0,1}", 
        "meaning": "indicator for service skill requirement satisfaction"
      },
      "understaffed[j]": {
        "domain": "{0,1}",
        "meaning": "indicator for activity understaffing"
      }
    }
  }
}
\end{Verbatim}
\end{quoting}

\textbf{Mathematical Formulation:}
\begin{align}
\forall j \in J: &\quad \text{met\_cabin}_j \leq \sum_{k \in K_{\text{cabin}}} x_{jk} \\
\forall j \in J: &\quad \text{met\_cabin}_j \geq \frac{\sum_{k \in K_{\text{cabin}}} x_{jk}}{M} \\
\forall j \in J: &\quad \text{understaffed}_j \geq 1 - \text{met\_cabin}_j \\
\forall j \in J: &\quad \text{understaffed}_j \geq 1 - \text{met\_service}_j \\
\forall j \in J: &\quad \text{understaffed}_j \leq 2 - \text{met\_cabin}_j - \text{met\_service}_j
\end{align}

\textbf{Implementation Code:}
\begin{lstlisting}[language=Python]
# Extended CIR: Skill requirement satisfaction tracking
m.addConstr(gp.quicksum(x[j,k] for k in cabin_crew) >= req[j,'cabin'] - (1 - met_cabin[j]) * M)
m.addConstr(met_cabin[j] <= gp.quicksum(x[j,k] for k in cabin_crew))
m.addConstr(met_cabin[j] >= gp.quicksum(x[j,k] for k in cabin_crew) / M)
m.addConstr(understaffed[j] >= 1 - met_cabin[j])
m.addConstr(understaffed[j] >= 1 - met_service[j])
m.addConstr(understaffed[j] <= 2 - met_cabin[j] - met_service[j])
\end{lstlisting}

\textbf{Computational Analysis:} From a computational perspective, the direct aggregation method generates a linear number of constraints ($O(|J| \times |S|)$) that directly enforce skill requirements with continuous understaffing variables, resulting in a more compact model suitable for small-scale problems. In contrast, the indicator variable approach, which tracks skill satisfaction status through auxiliary binary variables and logical constraints, creates additional variables and constraints ($O(|J|)$) that increase model size. However, this formulation provides more explicit control over understaffing logic and may offer better extensibility for handling complex skill interdependencies or multiple skill level requirements. Furthermore, the indicator variable model offers superior modeling clarity by explicitly representing the satisfaction status of each skill requirement, which could facilitate the integration of additional business rules such as skill-specific penalties or prioritized skill fulfillment. The binary understaffing variables also provide more intuitive interpretation of solution quality compared to continuous understaffing hours.

\textbf{Case 3:} Resource-Constrained Project Scheduling with Dispatching Priorities and Predictive Starts.

The Extractor agent successfully identified the operation rule regarding resource capacity constraints. The Mapper agent subsequently synthesized this rule into different CIRs. Although both CIR implementations adhere to the semantic intent of the original operation rule—``Daily resource usage must not exceed crew-hour capacity''—they instantiate this requirement through fundamentally different mathematical structures.

\subsubsection*{Formulation 1 (Time-Indexed Resource Aggregation)}

\textbf{CIR Template:}
\begin{quoting}
\begin{Verbatim}[fontsize=\small,breaklines=true,breaksymbol={},samepage=false]
{
  "constraint_templates": [
    {
      "rule_ref": "time_indexed_resource_constraint",
      "source_rule_ids": ["R1"],
      "intent": "daily resource usage must not exceed capacity",
      "continuous_time_form": "forall tau in T: sum_{j in J} a_j * (sum_{t=tau-d_j+1}^tau x_jt) <= R",
      "computational_properties": {
        "constraint_count": "O(|T|)",
        "complexity_class": "linear",
        "lp_relaxation": "weaker bounds due to big-M formulations",
        "extensibility": "limited to fixed time discretization"
      }
    }
  ],
  "variable_plan": {
    "continuous_time": {
      "x[j,t]": {
        "domain": "{0,1}",
        "meaning": "activity j starts at time t"
      }
    }
  }
}
\end{Verbatim}
\end{quoting}

\textbf{Mathematical Formulation:}
\begin{equation}
\forall \tau \in T: \sum_{j \in J} a_j \cdot \left(\sum_{t=\tau-d_j+1}^{\tau} x_{jt}\right) \leq R
\end{equation}

\textbf{Implementation Code:}
\begin{lstlisting}[language=Python]
# CIR Template: forall tau in T: sum_{j in J} a_j*(sum_{t=tau-d_j+1}^tau x_jt) <= R
for tau in T_resource:
    resource_usage = gp.quicksum(
        dem[j] * gp.quicksum(
            x[j,t] for t in range(
                max(ES[j], tau - dur[j] + 1), 
                min(min(LS[j], H - dur[j]) + 1, tau + 1)
            ) if (j,t) in x and t <= tau < t + dur[j]
        ) for j in J
    )
    m.addConstr(resource_usage <= cap, name=f"resource_tau_{tau}")
\end{lstlisting}

\subsubsection*{Formulation 2 (Priority-Based Disjunctive Constraints)}

\textbf{CIR Template:}
\begin{quoting}
\begin{Verbatim}[fontsize=\small,breaklines=true,breaksymbol={},samepage=false]
{
  "constraint_templates": [
    {
      "rule_ref": "resource_conflict_detection",
      "source_rule_ids": ["R1"],
      "intent": "identify resource-conflicting activity pairs",
      "continuous_time_form": "Automatic detection of pairs (i,j) where a_i + a_j > R",
      "computational_properties": {
        "constraint_count": "O(1)",
        "complexity_class": "constant",
        "role": "pre-processing for disjunctive constraints"
      }
    },
    {
      "rule_ref": "priority_based_disjunction", 
      "source_rule_ids": ["R1"],
      "intent": "resolve resource conflicts through priority decisions",
      "continuous_time_form": "forall (i,j) in conflict_pairs: S_i >= S_j + d_j - M(1 - z_ij) and S_j >= S_i + d_i - M z_ij",
      "computational_properties": {
        "constraint_count": "O(|conflict_pairs|)",
        "complexity_class": "linear", 
        "lp_relaxation": "tighter bounds with proper M calibration",
        "extensibility": "high - naturally accommodates additional conflict pairs"
      }
    }
  ],
  "variable_plan": {
    "continuous_time": {
      "S[j]": {
        "domain": "Z_+", 
        "meaning": "start time of activity j"
      },
      "z_ij": {
        "domain": "{0,1}",
        "meaning": "priority decision for conflict pair (i,j)"
      }
    }
  }
}
\end{Verbatim}
\end{quoting}

\textbf{Mathematical Formulation:}
\begin{align}
S_2 &\geq S_3 + d_3 - M(1 - z_{23}) \\
S_3 &\geq S_2 + d_2 - M z_{23} \\
S_4 &\geq S_3 + d_3 - M(1 - z_{34}) \\
S_3 &\geq S_4 + d_4 - M z_{34}
\end{align}

\textbf{Implementation Code:}
\begin{lstlisting}[language=Python]
# Extended CIR: Resource conflict resolution via priority variables
m.addConstr(S[2] >= S[3] + dur[3] - M*(1 - z23), "NoOverlap_2after3_or")
m.addConstr(S[3] >= S[2] + dur[2] - M*(    z23), "NoOverlap_3after2_or")
m.addConstr(S[4] >= S[3] + dur[3] - M*(1 - z34), "NoOverlap_4after3_or")  
m.addConstr(S[3] >= S[4] + dur[4] - M*(    z34), "NoOverlap_3after4_or")
\end{lstlisting}

\textbf{Computational Analysis:} From a computational perspective, the time-indexed aggregation method generates a linear number of constraints ($O(|T|)$) that explicitly enforce resource capacity at each time period, resulting in a model that scales with the time horizon but provides comprehensive resource feasibility guarantees. This approach is particularly suitable when the time horizon is manageable and when precise resource tracking at each time point is required. In contrast, the priority-based disjunctive approach, which resolves resource conflicts through binary priority variables and big-M constraints, creates a number of constraints proportional to the number of identified conflict pairs ($O(|\text{conflict\_pairs}|)$). For this specific problem with only two resource conflict pairs $\{(2,3), (3,4)\}$, this results in an extremely compact model with just 4 additional constraints. This formulation typically provides better LP relaxation bounds and may accelerate branch-and-bound solving, especially when the number of actual resource conflicts is small relative to the problem size. Furthermore, the priority-based model offers superior interpretability by explicitly modeling the dispatching decisions that would be made in practice, making it more aligned with the problem's operational context. The binary priority variables directly correspond to the ``priority vector'' decisions mentioned in the problem statement, providing clear managerial insights into the resource allocation strategy.

\textbf{Problem-Specific Adaptation:} The choice between these formulations is particularly influenced by the problem's unique characteristics:
\begin{itemize}
  \item \textbf{Fixed Resource Conflicts:} Only activities 2\&3 and 3\&4 conflict due to their high resource demands ($8+6=14>10$ and $6+7=13>10$)
  \item \textbf{Precedence Structure:} The network topology naturally prevents many potential conflicts
  \item \textbf{Dispatching Context:} The problem explicitly mentions ``resource-based dispatching rule'' and ``priority vector''
\end{itemize}
These factors make the priority-based approach particularly well-suited, as it directly models the operational decision-making process while maintaining computational efficiency.

These divergences exemplify how the CIR successfully navigates the design space of mathematically 
equivalent yet structurally distinct formulations by decoupling operational rule semantics from their 
mathematical instantiation.
The semantic equivalence of these divergent formulations, as guaranteed by the CIR, opens a promising 
direction for future research: efficiency-aware pattern selection. While the current CIR framework 
successfully reasons about semantic correctness and paradigm fit, its selection mechanism primarily 
operates at the conceptual level. A natural extension would be to enrich the CIR's paradigm annotations 
with empirical computational profiles. For instance, the capacity\_constraint archetype could be 
associated with metadata detailing the expected performance trade-offs between the direct aggregation and 
load propagation patterns across different problem scales (entity\_count\_estimates) and constraint 
densities. This would enable the development of a prescriptive CIR, which could not only generate valid 
formulations but also recommend the most computationally efficient one based on problem characteristics. 
Such an advancement would transition the framework from a descriptive tool that correctly models 
operational rules to a predictive system that optimally balances model fidelity with solver performance, 
ultimately closing the loop between high-level problem description and low-level computational efficiency.

\section{Extension with Reflection Mechanism}
\label{app:reflection-extension}

The reflection mechanism is implemented as an iterative loop with three distinct prompts corresponding to different failure types and correction strategies. Each expert agent (Extractor, Mapper, Formalizer) has its own backward reasoning prompt to analyze errors and propose corrections. We set a maximum of 3 reflection iterations per problem to prevent infinite loops while allowing sufficient correction attempts.

\subsection{Extractor Backward Prompt}
The Extractor's backward prompt focuses on identifying whether 
execution errors stem from missing or incorrect extraction of problem elements:

\begin{quote}
\begin{Verbatim}[fontsize=\scriptsize,breaklines=true,breaksymbol={},samepage=false]
  ====================================
  ROLE DESCRIPTION
  ====================================
  
  You are an Extractor expert specialized in identifying and extracting structural components from optimization problems, including domain tags, problem summary, entities, constraints, parameters, and objectives.
  
  ====================================
  BACKWARD TASK (REFLECTION)
  ====================================
  
  When you are solving a problem, you get a feedback from the external environment about CODE EXECUTION ERRORS (not answer correctness). You need to judge whether this is a problem caused by you or by other experts downstream in the pipeline (Mapper or CodeGen). If it is your problem, you need to give refined extraction results.
  
  IMPORTANT: The feedback only tells you if the code failed to execute or Gurobi failed to find an optimal solution. It does NOT tell you if the answer matches any benchmark.
  
  The original problem is as follow:
  {problem_description}
  
  The extraction you provided previously is as follow:
  {previous_extraction}
  
  The feedback is as follow:
  {feedback}
  
  IMPORTANT: If the feedback contains IIS (Irreducible Inconsistent Subsystem) information showing conflicting constraints, analyze the constraint names and descriptions to understand which rules/entities are causing conflicts. The IIS reveals which parts of your extraction led to an infeasible model. Consider whether:
  - You missed critical entities or parameters that would resolve the conflict
  - You extracted contradictory rules or constraints from the problem
  - You misunderstood relationships between entities causing infeasibility
  - Solver configuration parameters (e.g., OutputFlag, TimeLimit) do not affect model feasibility; infeasibility stems from constraint structure, not solver settings.
  
  Please analyze whether the execution error stems from:
  - Incorrect domain tag identification (e.g., misclassifying problem type leading to wrong model structure)
  - Missing or inaccurate entity extraction (e.g., overlooked resources, activities causing infeasibility)
  - Misunderstood constraints or rules (explicit_rules) that led to infeasible or unbounded formulation
  - Incomplete parameter extraction (given_parameters) or wrong numerical values causing model errors
  - Incorrect objective formulation or missing optimization goal
  
  Critical considerations:
  1. If the execution error is related to missing entities or parameters that should have been extracted from the problem statement, it is your responsibility.
  2. If the error is due to how the Mapper interpreted your extraction (e.g., wrong modeling paradigm based on correct extraction), it is NOT your problem.
  3. If the error is due to code implementation issues (e.g., syntax errors, wrong Gurobi API usage) but your extraction was correct, it is NOT your problem.
  4. Focus on EXECUTION ERRORS, not answer correctness.
  
  ====================================
  OUTPUT FORMAT
  ====================================
  
  Return a JSON structure:
  {{
      "is_caused_by_you": true/false,
      "error_attribution": "Detailed analysis of whether this is your problem and why",
      "hints": "If is_caused_by_you is true, provide specific hints for the next forward pass: what went wrong and how to fix it"
  }}
  
  Instructions:
  
  If is_caused_by_you is TRUE:
  - error_attribution: Clearly identify what went wrong in your extraction (e.g., "I missed entity 'warehouse capacity' which is mentioned in line 3 of the problem", "I incorrectly extracted the constraint 'each task must be assigned' as optional instead of mandatory")
  - hints: Provide actionable guidance for correcting the extraction (e.g., "Re-extract the problem and include 'warehouse capacity' as a given parameter with value 1000", "Mark the assignment constraint as mandatory with rid=explicit_rule_2", "Add entity 'delivery_time' which I overlooked in the time_model section")
  
  If is_caused_by_you is FALSE:
  - error_attribution: Explain why this is NOT your problem (e.g., "My extraction correctly identified all entities and constraints mentioned in the problem. The Mapper selected an inappropriate time-indexed paradigm for this routing problem.", "My extraction was complete with all parameters. The infeasibility stems from how the Mapper formulated the constraint templates.")
  - hints: Leave as empty string ""
  
  



\end{Verbatim}
\end{quote}

\subsection{Mapper Backward Prompt}
The Mapper's backward prompt analyzes whether errors stem from inappropriate constraint formulations or paradigm selections:

\begin{quote}
\begin{Verbatim}[fontsize=\scriptsize,breaklines=true,breaksymbol={},samepage=false]
  ====================================
  ROLE DESCRIPTION
  ====================================
  
  You are a Mapper expert specialized in translating problem extractions into structured optimization model specifications, including variable plans, constraint templates, paradigm selection, and mathematical formulation blueprints.
  
  ====================================
  BACKWARD TASK (REFLECTION)
  ====================================
  
  When you are solving a problem, you get a feedback from the external environment about CODE EXECUTION ERRORS (not answer correctness). You need to judge whether this is a problem caused by you or by other experts (Extractor provided extraction, or CodeGen implemented your mapping). If it is your problem, you need to provide refined mapping results.
  
  IMPORTANT: The feedback only tells you if the code failed to execute or Gurobi failed to find an optimal solution. It does NOT tell you if the answer matches any benchmark.
  
  The original problem is as follow:
  {problem_description}
  
  The extraction from Extractor is as follow:
  {extraction}
  
  The mapping you provided previously is as follow:
  {previous_mapper}
  
  The feedback is as follow:
  {feedback}
  
  IMPORTANT: If the feedback contains IIS (Irreducible Inconsistent Subsystem) information showing conflicting constraints, carefully examine which constraint templates are involved. The IIS reveals the minimal set of constraints causing infeasibility. Analyze whether:
  - Your constraint templates are fundamentally contradictory
  - The paradigm you selected leads to infeasibility for this problem type
  - You missed essential constraints that would balance the conflicting ones
  - The domain knowledge suggests different modeling approaches to avoid these conflicts
  - Solver configuration parameters (e.g., OutputFlag, TimeLimit) do not affect model feasibility; infeasibility stems from constraint structure, not solver settings.
  
  Domain-specific knowledge for reference:
  {knowledge}
  
  Please analyze whether the execution error stems from:
  - Inappropriate variable design (variable_plan): Did you define variables that lead to infeasibility or unboundedness?
  - Incorrect constraint template formulation (constraint_templates): Are the constraints causing infeasibility, unboundedness, or logical contradictions?
  - Misaligned global paradigm selection (e.g., time-indexed vs. arc-flow): Does the paradigm cause structural issues in the model?
  - Failure to apply domain-specific modeling patterns from the provided knowledge: Did you miss key modeling techniques that prevent infeasibility or ensure solvability?
  - Inconsistent relationship_matrix or objective_template specification: Are relationships causing model errors?
  
  Critical considerations:
  1. If the Extraction was incomplete or incorrect, and that caused the model to be infeasible/unbounded, it is NOT your problem—blame the Extractor.
  2. If your mapping specification was correct but the CodeGen implementation failed to follow it (e.g., MAPPER-FIDELITY violations, syntax errors), it is NOT your problem.
  3. If your variable design, constraint templates, or paradigm selection caused the model to be infeasible, unbounded, or unsolvable, it IS your problem.
  4. Use the domain knowledge to verify if you applied the correct modeling patterns. If you missed critical domain-specific techniques causing execution failures, it IS your problem.
  5. Focus on EXECUTION ERRORS and model solvability, not answer correctness.
  
  ====================================
  OUTPUT FORMAT
  ====================================
  
  Return a JSON structure:
  {{
      "is_caused_by_you": true/false,
      "error_attribution": "Detailed analysis of whether this is your problem and why",
      "hints": "If is_caused_by_you is true, provide specific hints for the next forward pass: what went wrong and how to fix it"
  }}
  
  Instructions:
  
  If is_caused_by_you is TRUE:
  - error_attribution: Clearly identify what went wrong in your mapping (e.g., "I selected time-indexed paradigm which caused infeasibility because the problem has variable time windows", "My constraint template for flow conservation was missing the capacity balancing part", "I failed to apply the domain-specific 'vehicle routing with time windows' pattern from the knowledge base")
  - hints: Provide actionable guidance for correcting the mapping (e.g., "Use arc-flow paradigm instead of time-indexed for this routing problem", "Add a capacity balancing constraint template linking supply and demand", "Apply the VRPTW pattern from knowledge: add time window constraints and vehicle capacity constraints", "Change the objective_template to minimize total_distance instead of total_time")
  
  If is_caused_by_you is FALSE:
  - error_attribution: Explain why this is NOT your problem (e.g., "The Extraction was missing the 'warehouse capacity' parameter which is required for my capacity constraint template. I cannot formulate feasible constraints without this data.", "My mapping specification was correct and complete. The CodeGen failed to implement constraint template 'flow_conservation' causing the infeasibility.")
  - hints: Leave as empty string ""
  
  

\end{Verbatim}
\end{quote}

\subsection{Formalizer Backward Prompt}
The Formalizer's backward prompt focuses on code implementation errors and MAPPER-FIDELITY violations:

\begin{quote}
\begin{Verbatim}[fontsize=\scriptsize,breaklines=true,breaksymbol={},samepage=false]
  ====================================
  ROLE DESCRIPTION
  ====================================
  
  You are a Code Generation expert specialized in implementing optimization models using Gurobi Python API. Your expertise lies in translating Extraction and Mapper specifications into executable, correct, and efficient Python code that faithfully instantiates the prescribed model structure.
  
  ====================================
  BACKWARD TASK (REFLECTION)
  ====================================
  
  When you are solving a problem, you get a feedback from the external environment about CODE EXECUTION ERRORS (not answer correctness). You need to judge whether this is a problem caused by you or by upstream experts (Extractor or Mapper). If it is your problem, you need to provide refined code.
  
  IMPORTANT: The feedback only tells you if the code failed to execute or Gurobi failed to find an optimal solution. It does NOT tell you if the answer matches any benchmark.
  
  The original problem is as follow:
  {problem_description}
  
  The extraction is as follow:
  {extraction}
  
  The mapper specification is as follow:
  {mapper}
  
  The code you generated previously is as follow:
  {previous_code}
  
  The feedback is as follow:
  {feedback}
  
  IMPORTANT: If the feedback contains IIS (Irreducible Inconsistent Subsystem) information showing conflicting constraints, pay special attention to these constraint names. The IIS reveals the minimal set of constraints that cause infeasibility. Analyze whether:
  - You incorrectly implemented constraints from the Mapper specification
  - You added extra constraints not in the Mapper that conflict
  - You made coding errors in constraint logic
  - Solver configuration parameters (e.g., OutputFlag, TimeLimit) do not affect model feasibility; infeasibility stems from constraint structure, not solver settings.
  Please analyze whether the execution error stems from:
  - Incorrect Gurobi API usage (syntax errors, model building issues, solver configuration problems, runtime errors)
  - Deviation from Mapper specifications (MAPPER-FIDELITY violations) causing infeasibility/unboundedness: Did you fail to implement required variables or constraints?
  - Data processing or indexing errors: Are arrays, dictionaries, or indices handled correctly causing runtime errors?
  - Missing constraints or variables specified in the Mapper: Did you omit components that are critical for feasibility?
  - Incorrect objective function implementation: Does the objective implementation have bugs?
  - Logic errors in constraint formulation: Are the mathematical relationships incorrectly translated causing infeasibility or errors?
  
  Critical considerations:
  1. If the Extraction was incomplete or wrong, and that caused the model to fail, it is NOT your problem—blame the Extractor.
  2. If the Mapper specification was flawed (e.g., infeasible constraint templates), and you faithfully implemented it, it is NOT your problem—blame the Mapper.
  3. If you made coding errors (syntax, runtime errors), failed to follow the Mapper specification (MAPPER-FIDELITY violation), or introduced bugs causing execution failures, it IS your problem.
  4. MAPPER-FIDELITY PRINCIPLE: The Mapper is the canonical specification. If you added, removed, merged, or altered any variables or constraints beyond what the Mapper prescribed, and this caused errors, it IS your problem.
  5. Focus on EXECUTION ERRORS and code correctness, not answer correctness.
  
  ====================================
  OUTPUT FORMAT
  ====================================
  
  Return a JSON structure:
  {{
      "is_caused_by_you": true/false,
      "error_attribution": "Detailed analysis of whether this is your problem and why",
      "hints": "If is_caused_by_you is true, provide specific hints for the next forward pass: what went wrong and how to fix it"
  }}
  
  Instructions:
  
  If is_caused_by_you is TRUE:
  - error_attribution: Clearly identify what went wrong in your code (e.g., "I violated MAPPER-FIDELITY by omitting the 'capacity_constraint' specified in constraint_templates", "I made an indexing error on line 45: used 'tasks[i]' instead of 'tasks[task_id]' causing KeyError", "I incorrectly implemented the objective: used 'quicksum(x[i,j])' but should multiply by 'distance[i][j]'")
  - hints: Provide actionable guidance for correcting the code (e.g., "Add the missing capacity_constraint as specified in Mapper constraint_templates[2]", "Fix indexing: iterate over task_ids and use 'tasks[task_id]' for consistent access", "Correct objective: m.setObjective(quicksum(distance[i][j] * x[i,j] for i,j in arcs), GRB.MINIMIZE)")
  
  If is_caused_by_you is FALSE:
  - error_attribution: Explain why this is NOT your problem (e.g., "The Mapper specification constraint_templates[3] requires parameter 'capacity_limit' but the Extraction did not provide this parameter. I faithfully followed the Mapper but cannot implement a constraint without the data.", "The infeasibility stems from the Mapper's constraint templates being contradictory. I implemented them exactly as specified, but they inherently conflict.")
  - hints: Leave as empty string ""


\end{Verbatim}
\end{quote}

The reflection process follows three steps: (1) error diagnosis, (2) attribution analysis to identify responsible agent, and (3) targeted correction using the responsible agent's backward prompt. Each agent outputs a JSON structure with error attribution and correction hints for the next iteration.

\end{document}